\pdfoutput=1
\def\year{2021}\relax
\documentclass[letterpaper]{article} 
\usepackage{aaai21}  
\usepackage{times}  
\usepackage{helvet} 
\usepackage{courier}  
\usepackage[hyphens]{url}  
\usepackage{graphicx} 
\usepackage{natbib}  
\usepackage{caption} 
\usepackage[dvipsnames]{xcolor}
\usepackage[utf8]{inputenc}
\usepackage{soul}
\usepackage{booktabs}
\usepackage[switch]{lineno}

\usepackage{amsmath}
\usepackage{amssymb}
\usepackage{mathtools}
\usepackage{amsthm}
\newtheorem{theorem}{Theorem}
\theoremstyle{definition}
\newtheorem{definition}{Definition}

\usepackage[ruled,noend]{algorithm2e}
\usepackage{algpseudocode}

\usepackage{pgfplots}
\pgfplotsset{compat=1.15}
\usepackage{pgfplotstable}
\usepackage{filecontents}
\usepgfplotslibrary{statistics}
\usepgfplotslibrary{patchplots}
\usepackage{csvsimple}
\usepackage{multicol}
\usepackage{multirow}

\pgfplotsset{
    select coords between index/.style 2 args={
        x filter/.code={
            \ifnum\coordindex<#1\fi
            \ifnum\coordindex>#2\fi
        }
    },
    rshift/.style={
        xshift=\pgfkeysvalueof{/pgfplots/rshift scale}
    },
    lshift/.style={
        xshift=-\pgfkeysvalueof{/pgfplots/lshift scale}
    },
    rshift2/.style={
        xshift=\pgfkeysvalueof{/pgfplots/rshift2 scale}
    },
    lshift2/.style={
        xshift=-\pgfkeysvalueof{/pgfplots/lshift2 scale}
    },
    rshift2 scale/.initial=0.4em,
    rshift scale/.initial=0.2em,
    lshift scale/.initial=0.2em,
    lshift2 scale/.initial=0.4em,
}

\usepackage[capitalise,noabbrev]{cleveref}

\newcommand{\sectionClearPage}[0]{}

\usepackage[textsize=tiny,textwidth=1.5cm,disable]{todonotes}

\newcommand{\Aron}[1]{\todo[color=yellow!5,linecolor=black!50]{\textbf{Aron}: #1}}
\newcommand{\iAron}[1]{\todo[color=yellow!5,linecolor=black!50,inline]{\textbf{Aron}: #1}}

\newcommand{\ad}[1]{\todo[color=green!5,linecolor=black!50]{\textbf{AD}: #1}}

\newcommand{\calA}[0]{\ensuremath{\mathcal{A}}}
\newcommand{\calC}[0]{\ensuremath{\mathcal{C}}}
\newcommand{\calE}[0]{\ensuremath{\mathcal{E}}}
\newcommand{\calL}[0]{\ensuremath{\mathcal{L}}}
\newcommand{\calM}[0]{\ensuremath{\mathcal{M}}}
\newcommand{\calN}[0]{\ensuremath{\mathcal{N}}}
\newcommand{\calS}[0]{\ensuremath{\mathcal{S}}}
\newcommand{\calT}[0]{\ensuremath{\mathcal{T}}}
\newcommand{\calX}[0]{\ensuremath{\mathcal{X}}}
\newcommand{\calV}[0]{\ensuremath{\mathcal{V}}}

\newcommand{\MT}[0]{\ensuremath{T}}
\newcommand{\DMT}[2]{\ensuremath{D({#1}, {#2})}}

\newcommand{\argmax}[0]{\ensuremath{\operatorname{argmax}}}
\newcommand{\argmin}[0]{\ensuremath{\operatorname{argmin}}}

\newenvironment{revision}{\color{blue}}{\color{black}}
\renewenvironment{revision}{}{}

\newif\ifExtendedVersion
\ExtendedVersiontrue

\title{Minimizing Energy Use of Mixed-Fleet Public Transit for Fixed-Route Service}

\author{
Amutheezan Sivagnanam\textsuperscript{\rm 1},
Afiya Ayman\textsuperscript{\rm 1},
Michael Wilbur\textsuperscript{\rm 2},\\
Philip Pugliese\textsuperscript{\rm 3},
Abhishek Dubey\textsuperscript{\rm 2},
Aron Laszka\textsuperscript{\rm 1}\\
}

\affiliations{
\textsuperscript{\rm 1}University of Houston,
\textsuperscript{\rm 2}Vanderbilt University,
\textsuperscript{\rm 3}Chattanooga Area Regional Transportation Authority
}

\begin{document}

\setlength{\marginparwidth}{1.5cm}
\pagestyle{plain}

\maketitle

\ifExtendedVersion
\begin{center}
Accepted for publication in the proceedings of the 35th AAAI Conference on Artificial Intelligence (AAAI-21).
\end{center}\vspace{1em}
\fi

\begin{abstract}
Affordable public transit services are crucial for communities since they enable residents to access employment, education, and other services. 
Unfortunately, transit services that provide wide coverage tend to suffer from relatively low utilization, which results in high fuel usage per passenger per mile, leading to high operating costs and environmental impact.
Electric vehicles (EVs) can reduce energy costs and environmental impact, but most public transit agencies have to employ them in combination with conventional, internal-combustion engine vehicles due to the high upfront costs of EVs. 
To make the best use of such a mixed fleet of vehicles, transit agencies need to optimize route assignments and charging schedules, which presents a challenging problem for large transit networks. 
We introduce a novel problem formulation to minimize fuel and electricity use by assigning vehicles to transit trips and scheduling them for charging, while serving an existing fixed-route transit schedule. 
We present an integer program for optimal assignment and scheduling, and we propose polynomial-time heuristic and meta-heuristic algorithms for larger networks. 
We evaluate our algorithms on the public transit service of \begin{revision}Chattanooga, TN\end{revision} using operational data collected from transit vehicles. 
Our results show that the proposed algorithms are scalable and can reduce energy use and, hence, environmental impact and operational costs. 
For \begin{revision}Chattanooga\end{revision}, the proposed algorithms can save \begin{revision}\$145,635\end{revision} in energy costs and \begin{revision}576.7\end{revision} metric tons of CO$_2$ emission annually.
\end{abstract}

\section{Introduction}
\label{sec:intro}

Affordable public transit services are the backbones of many communities, providing diverse groups of people with access to employment, education, and other services.
Affordable transit services are especially important in low-income neighborhoods where residents might not be able to afford personal vehicles.
However, transit services that provide wide and equitable coverage tend to suffer from low utilization---compared to concentrating service into a few high-density areas---which leads to higher fuel usage per passenger per mile.
This in turn results in higher operating costs, which  threatens affordability---a problem that has recently been exacerbated by the COVID-19 pandemic.


Further, low utilization also leads to high environmental impact per passenger per mile.
In the U.S.,
28\% of total energy use is for transportation
~\cite{eia}. 
While public transit services can be very energy-efficient compared to personal vehicles, their environmental impact is significant nonetheless.
For example, bus transit services in the U.S. may be responsible for up to 21.1 million metric tons of CO$_2$ \cite{ghgemissions} emission every year.
\Aron{this \url{https://nepis.epa.gov/Exe/ZyPDF.cgi?Dockey=P100WUHR.pdf} reference says 19.7 million tons, which one is right?}


Electric vehicles (EVs) can have much lower operating costs and lower environmental impact during operation than comparable internal combustion engine vehicles (ICEVs), especially in urban areas.
Unfortunately, EVs are also much more expensive than ICEVs: typically, diesel transit buses cost less than \$500K, while electric ones cost more than \$700K, or close to around \$1M with charging infrastructure. 
As a result, many public transit agencies can afford only mixed fleets of transit vehicles, which may consist of EVs, hybrid vehicles (HEVs), and ICEVs. 

Public transit agencies that operate such mixed fleets of vehicles face a challenging optimization problem.
First, they need to decide which vehicles are assigned to serving which transit trips.
Since the advantage of EVs over ICEVs varies depending on the route and time of day (e.g.,  advantage of EVs is higher in slower traffic with frequent stops, and lower on highways), the assignment can have a significant impact on energy use and, hence, on costs and environmental impact.  
Second, transit agencies need to schedule when to charge electric vehicles because EVs have limited battery capacity and driving range, and may need to be recharged during the day between serving transit trips. 
Since agencies often have limited charging capabilities (e.g., limited number of charging poles, or limited maximum power to avoid high peak loads on the electric grid), charging constraints can significantly increase the complexity of the assignment and scheduling problem. 

\textbf{Contributions:} While an increasing number of transit agencies face these problems, there exist no practical solutions to the best of our knowledge. 
In this paper, we present a novel problem formulation and algorithms for assigning a mixed fleet of transit vehicles to trips and for scheduling the charging of electric vehicles.
We developed this problem formulation in collaboration with \begin{revision}the Chattanooga Area Regional Transportation Authority (CARTA), the public transit agency of Chattanooga, TN\end{revision}, which operates a fleet of EVs, HEVs, and ICEVs.
To solve the problem, we introduce an integer program as well as greedy and simulated annealing algorithms.
We evaluate these algorithms using operational data collected from \begin{revision}CARTA\end{revision} (e.g., vehicle energy consumption data, transit routes and schedules) and from other sources (e.g., elevation and street maps). 
Based on our numerical results, the proposed algorithms can reduce
energy costs by up to \begin{revision} \$145,635 \end{revision} and CO$_2$ emissions by
up to \begin{revision} 576.7 \end{revision} metric tons annually for~\begin{revision}CARTA\end{revision}. 
We will make all data and the implementation of our algorithms publicly available (also attached as appendix to this submission).

Our problem formulation applies to a wide range of public transit agencies that operate fixed-route services.
Our objective is to minimize energy consumption (i.e., fuel and electricity use), which leads to lower operating costs and environmental impact---as demonstrated by our numerical results. 
Our problem formulation considers assigning and scheduling for a single day (it may be applied to any number of consecutive days one-by-one), and allows any physically possible re-assignment during the day.
Our formulation also allows capturing additional constraints on charging; for example,~\begin{revision}CARTA\end{revision} aims to charge only one vehicle at a time to avoid demand charges from the electric utility. 

\textbf{Organization:} 
In \cref{sec:model}, we describe our model and problem formulation.
In \cref{sec:algo}, we introduce a mixed-integer program as well as greedy and simulated annealing algorithms.
In \cref{sec:numerical}, we provide numerical results based on real-world data from~\begin{revision}CARTA\end{revision}.
In \cref{sec:related}, we present a brief overview of related work.
Finally, in \cref{sec:concl}, we summarize our findings and outline future work.
\Aron{check if organization is accurate once the paper is finished}

\sectionClearPage
\section{Transit Model and Problem Formulation}
\label{sec:model}


\paragraph{Vehicles}
We consider a transit agency that operates a set of \emph{buses}~$\mathcal{V}$.
Note that we will use the terms \emph{bus} and \emph{vehicle} interchangeably.
Each bus $v \in \mathcal{V}$ belongs to a \emph{vehicle model}~$M_v \in \mathcal{M}$, where $\mathcal{M}$ is the set of all vehicle models in operation.
We divide the set of vehicle models $\mathcal{M}$ into two disjoint subsets: liquid-fuel models $\mathcal{M}^\text{gas}$ (e.g., diesel, hybrid), and electric models $\mathcal{M}^\text{elec}$.
Based on discussions with~\begin{revision}CARTA\end{revision}, we assume that vehicles belonging to a liquid-fuel model 
can operate all day without refueling.
On the other hand, vehicles belonging to an electric model 
have limited battery capacity, which might not be enough for a whole day.
For each electric vehicle model $m \in \mathcal{M}^\text{elec}$, 
we let~$C_m$ denote the \emph{battery capacity} of vehicles of model~$m$.

\paragraph{Locations} 
Locations 
$\mathcal{L}$ include 
bus stops, garages, and charging stations in the transit network. 

\paragraph{Trips}
During the day, the agency has to serve a given set of \emph{transit trips}~$\mathcal{T}$ using its buses. 
Based on discussions with~\begin{revision}CARTA\end{revision}, we assume that locations and time schedules are fixed for every trip. 
A bus serving trip $t \in \mathcal{T}$ leaves  from trip origin $t^{\text{origin}} \in \mathcal{L}$ at time $t^{\text{start}}$ and arrives at destination $t^{\text{destination}} \in \mathcal{L}$ at time $t^{\text{end}}$. 
Between $t^{\text{origin}}$ and $t^{\text{destination}}$, the bus must pass through a series of stops at fixed times; however, since we cannot re-assign a bus during a transit trip, the locations and times of these stops are inconsequential to our model.
Finally, we assume that any bus may serve any trip. Note that it would be straightforward to extend our model and algorithms to consider constraints on which buses may serve a trip (e.g., based on passenger capacity).

\paragraph{Charging}
To charge its electric buses,
the agency operates a set of \emph{charging poles} $\mathcal{CP}$, which are typically located at bus garages or charging stations in practice.
We let $cp^\text{location} \in \calL$ denote the location of charging pole $cp \in \mathcal{CP}$.

For the sake of computational tractability, we use a discrete-time model to schedule charging, which divides time into uniform-length \emph{time slots} $\mathcal{S}$.\Aron{this is kind sloppy since we don't enforce the time slots to cover the duration of the entire schedule, which is needed by later formulations (should be kinda obvious, but not necessarily} 
A time slot $s \in \mathcal{S}$ begins at~$s^{\text{start}}$ and ends at~$s^{\text{end}}$. 
\Aron{can we include this in the appendix? in other words, can we have these results next week? if yes, then change reference to \cref{app:numerical}; if no, then comment out this sentence}
A charging pole~$cp \in \mathcal{CP}$ can charge $P(cp, M_v)$ energy to one electric bus $v$ in one time slot.
We will refer to the combination of a charging pole $cp \in \mathcal{CP}$ and a time slot $s \in \mathcal{S}$ as a \emph{charging slot} $(cp, s)$;
and we let $\mathcal{C} = \mathcal{CP} \times \mathcal{S}$ denote the set of charging slots.

\paragraph{Non-Service Trips}
Besides serving transit trips, buses may also need to drive between trips or charging poles.
For example, if a bus has to serve a trip that starts from a location that is different from the destination of the previous trip, the bus first needs to drive to the origin of the subsequent trip.
An electric bus may also need to drive to a charging pole after serving a transit trip to recharge, and then drive from the pole to the origin of the next transit trip.
We will refer to these deadhead trips, which are driven outside of revenue service, as \emph{non-service trips}.
%
We let $\MT(l_1, l_2)$ denote the non-service trip from location $l_1 \in \mathcal{L}$ to $l_2 \in \mathcal{L}$;
and we let $\DMT{l_1}{l_2}$ denote the time duration of this non-service trip.



\subsection{Solution Space}
Our primary goal is to assign a bus to each transit trip. 
Additionally, electric buses may also
need to be assigned to charging slots to prevent them from running out of power. 

\paragraph{Solution Representation}
We represent a solution as a \emph{set of assignments} $\calA$. 
For each trip $t \in \mathcal{T}$, a solution assigns exactly one bus $v \in \mathcal{V}$ to serve trip $t$; this assignment is represented by the relation \Aron{make sure that we use this notation everywhere consistently}
$\langle v,t \rangle \in \calA$.
Secondly, each electric bus $v$ must be charged before its battery state of charge drops below the safe level for operation.
A solution assigns at most one electric bus $v$ to each charging slot $(cp, s) \in \mathcal{C}$; this assignment is represented by the relation \Aron{$\langle v, (cp, s) \rangle \in \calA$ ?!}$\langle v, (cp, s) \rangle \in \calA$.
We assume that when a bus is assigned for charging, it remains at the charging pole for the entire duration of the corresponding time slot.

\paragraph{Constraints}
If a bus $v$ is assigned to serve an earlier transit trip $t_1$ and a later trip $t_2$, then the duration of the non-service trip from $t_1^\text{destination}$ to $t_2^\text{origin}$ must be less than or equal to the time between $t_1^\text{end}$ and $t_2^\text{start}$.
Otherwise, it would not be possible to serve $t_2$ on time.
We formulate this constraint~as:
\begin{align}
\forall t_1, t_2 \in \calT; ~ t_1^\text{start} \leq t_2^\text{start}; ~ \langle v,t_1 \rangle \in \calA ;  ~ \langle v,t_2 \rangle \in \calA : \nonumber \\
t_1^{\text{end}} + \DMT{t_1^{\text{destination}}}{t_2^{\text{origin}}} \leq t_2^{\text{start}}
\label{equ:constraint_1a}
\end{align}
Note that if the constraint is satisfied by every pair of consecutive trips assigned to a bus, then it is also satisfied by every pair of non-consecutive trips assigned to the bus. 

We need to formulate similar constraints for non-service trips to, from, and between charging slots:
\begin{align}
\forall t \in \calT; & ~ (cp, s) \in \calC; \, t^\text{start} \leq s^\text{start}; \,  \langle v,t \rangle, \langle v, (cp,s) \rangle \in \calA : \nonumber \\
& t^{\text{end}} +\DMT{t^{\text{destination}}}{cp^{\text{location}}} \leq s^{\text{start}}
\label{equ:constraint_1b} \\
\forall t \in \calT; & ~ (cp, s) \in \calC; t^\text{start} \geq s^\text{start};  \langle v,t \rangle, \langle v, (cp,s) \rangle \in \calA : \nonumber \\
& s^{\text{end}} +\DMT{{cp^{\text{location}}}}{{t^{\text{origin}}}} \leq t^{\text{start}}
\label{equ:constraint_1c} 
\end{align}
\begin{scriptsize}
\begin{align}
\forall (cp_1, s_1), & ~ (cp_2, s_2) \in \calC;  s_1^\text{start} \leq s_2^\text{start};   \langle v,(cp_1, s_1) \rangle, \langle v, (cp_2,s_2) \rangle \in \calA : \nonumber \\
& s_1^{\text{end}} + \DMT{cp_1^{\text{location}}}{cp_2^{\text{location}}} \leq s_2^{\text{start}}
\label{equ:constraint_1d}
\end{align}
\end{scriptsize}
\Aron{fix spacing once paper is ready}
\Aron{check back on improving formatting if we have extra time}


We also need to ensure that electric buses never run out of power.
First, we let $\calN(\calA, v, s)$ denote the set of all non-service trips that bus $v$ needs to complete by the end of time slot~$s$ according to the set of assignments $\calA$.
\Aron{note: in ACM TOIT submission, we can express this formally (here we don't have space)}
In other words, $\calN(\calA, v, s)$ is the set of all necessary non-service trips to the origins of transit trips that start by $s^\text{end}$ and to the locations of charging slots that start by $s^\text{end}$.
Next, we let \Aron{this should depend on the vehicle model, not the actual vehicle} $E(v, t)$ denote the amount of energy used by bus $v$ to drive a transit or non-service trip~$t$.
Then, we let $e(\calA, v, s)$ be the amount of energy used by bus~$v$ for all trips completed by the end of time slot $s$: 
\begin{equation}
e(\calA, v, s) = \sum_{t \in \calN(\calA, v, s)} E(v, t) + \quad \quad \sum_{\mathclap{t \in \calT, \, \langle v, t \rangle \in \calA , \, t^\text{end} \leq s^\text{end}}}  \quad ~~ E(v, t)    
\end{equation}
Similarly, we let $r(\calA, v, s)$ be the amount of energy charged to bus $v$ by the end of time slot $s$:
\begin{equation}
r(\calA, v, s) = \sum_{(cp, \hat{s}) \in \calC, \, \langle v, (cp,\hat{s}) \rangle \in \calA, \, \hat{s}^\text{end} \leq s^\text{end} } P(cp, M_v)
\end{equation}
Since a bus can be assigned for charging only to complete time slots, the minima and maxima of its battery level will be reached at the end of time slots.
Therefore, we can express the constraint that the battery level of bus $v$ must always remain between $0$ and the battery capacity $C_{M_v}$ as
\begin{equation}
    \forall v \in \calV, \forall s \in \calS: ~ 0 < r(\calA, v, s) - e(\calA, v, s) \leq C_{M_v} .
    \label{eq:energy_constr}
\end{equation}
Note that we can give vehicles an initial battery charge by adding ``virtual'' charging slots before the day starts.

\subsection{Objective}
Our objective is to minimize the energy use of the transit vehicles.
We can use this objective to minimize both environmental impact and operating costs by imposing the appropriate cost factors on the energy use of liquid-fuel and electric vehicles. 
We let $K^\text{gas}$ and $K^\text{elec}$ denote the unit costs of energy use for liquid-fuel and electric vehicles, respectively.
Then, by applying the earlier notation $e(\calA, v, s)$ to all vehicles, we can express our objective as 
\begin{equation}
\min_{\calA} ~~\quad \sum_{\mathclap{v \in \calV: \, M_v \in \calM^\text{gas}}} ~ K^\text{gas} \! \cdot e(\calA, v, s_\infty) + ~~\quad \sum_{\mathclap{v \in \calV: \, M_v \in \calM^\text{elec}}} ~ K^\text{elec} \! \cdot e(\calA, v, s_\infty)
\label{equ:obj_main}
\end{equation}
where $s_\infty$ denotes the last time slot of the day.

\sectionClearPage
\section{Algorithms}
\label{sec:algo}

Since this optimization problem is computationally hard (see proof sketch in \begin{revision}
\ifExtendedVersion
\cref{app:complexity}),
\else
an online appendix~\cite{sivagnanam2020minimizing}),
\fi
\end{revision}we first present an integer program to find optimal solutions for smaller instances (\cref{sec:ip}).
Then, 
we introduce efficient greedy (\cref{sec:greedy}) and simulated annealing algorithms (\cref{sec:annealing}),
which scale well for larger instances. \begin{revision}
Due to lack of space, we include less important subroutines in
\ifExtendedVersion
 \cref{app:algorithms}.
\else
the online appendix~\cite{sivagnanam2020minimizing}.
\fi
\end{revision}



\subsection{Integer Program}
\label{sec:ip}



\paragraph{Variables}
Our integer program has five sets of variables.
Three of them are binary to indicate assignments and non-service trips.
First, $a_{v,t} = 1$ (or $0$) indicates that trip $t$ is assigned to bus $v$ (or that it is not). 
Second, $a_{v, (cp,s)} = 1$ (or $0$) indicates that charging slot $(cp,s)$ is assigned to electric bus $v$ (or not). 
Third, $m_{v,x_1,x_2} = 1$ (or $0$) indicates that bus $v$ takes the non-service trip between a pair  $x_1$ and $x_2$ of transit trips and/or charging slots (or not). 
Note that for requiring non-service trips (see \cref{equ:constraint_1a,equ:constraint_1b,equ:constraint_1c,equ:constraint_1d}), we will treat transit trips and charging slots similarly since they induce analogous constraints.
There are also two sets of continuous variables.
First, $c^v_s \in [0, C_{M_v}]$ represents the amount of energy charged to electric bus $v$ in time slot $s$. 
Second, $e^v_s \in [0, C_{M_v}]$ represents the battery level of electric bus $v$ at the start of time slot~$s$ (considering energy use only for trips that have ended by that time).
Due to the continuous variables, our program is a mixed-integer program.

\paragraph{Constraints}
First, we ensure that every transit trip is served by exactly one bus:
\begin{equation*}
\forall t \in \calT: ~ \sum_{v \in \mathcal{V}} a_{v,t} = 1
\end{equation*}
Second, we ensure that each charging slot is assigned at most one electric vehicle: 
\begin{equation*}
    \forall (cp,s) \in \calC: \sum_{ \forall v \in \calV: ~  M_v \in \mathcal{M}^{\text{elec}}} a_{v,(cp,s)} \leq 1
\end{equation*}

Next, we ensure that \cref{equ:constraint_1a,equ:constraint_1b,equ:constraint_1c,equ:constraint_1d} are satisfied.
We let $F(x_1, x_2)$ be \emph{true} if a pair $x_1, x_2$ of transit trips and/or charging slots satisfies the applicable one from \cref{equ:constraint_1a,equ:constraint_1b,equ:constraint_1c,equ:constraint_1d}; and  let it be \emph{false} otherwise.
Then, we can express these constraints as follows:
\begin{equation*}
\forall v \in \calV, \forall x_{1}, x_{2}, \lnot F(x_1, x_2) : ~ a_{v,x_1} + a_{v, x_2} \leq 1
\end{equation*}

When a bus $v$ is assigned to both $x_1$ and $x_2$, but it is not assigned to any other transit trips or charging slots in between (i.e., if $x_1$ and $x_2$ are consecutive assignments), then bus~$v$ needs to take a non-service trip: 
\begin{equation*}
m_{v, x_1, x_2} \geq a_{v, x_1} + a_{v, x_2} - 1
    -  {\sum_{\substack{x \in \calT \,\cup\, \calC:~ x_1^\text{start} \leq x^\text{start} \leq x_2^\text{start}}}}  a_{v,x}
\end{equation*}
Note that if $x_1$ ends at the same location where $x_2$ starts, then the non-service trip will take zero time and energy.

Finally, we ensure that the battery levels of electric buses remain between zero and capacity.
First, for each slot $s$ and electric bus $v$, the amount of energy charged $c^v_s$ is subject to
\begin{equation*}
    c^v_{s} \leq \sum_{(cp,s) \in \calC} a_{v, (cp,s)} \cdot P(cp, M_v) .
\end{equation*}
Then, for the $(n+1)$\textit{th} time slot $s_{n+1}$ and for an electric bus $v$, we  can  express variable $e_{s_{n+1}}^v$~as 
\Aron{equation to wide, break to two lines?}
\begin{align*}
e^v_{s_{n+1}} \!\!= e^v_{s_n} \!\! + 
       c^v_{s_n} \!\!&- ~~~~~~~~~ {\sum_{\mathclap{\substack{t \in \calT : \, s_n^\text{start} < t^\text{end} \leq s_n^\text{end}}}}} ~~~~~~~~~ a_{v,t} \! \cdot \! E(v,t) \\
     &- ~~~~~~~~~ {\sum_{ \mathclap{\substack{ x_1, x_2 : \, s_n^\text{start} < x_2^\text{start} \leq s_n^\text{end} }}}} ~~~~~~~~~ m_{v, x_1, x_2} \! \cdot \! E(v, \MT(x_1, x_2))
\end{align*}
%
where $s_{n}$ is the $(n)$\textit{th}  slot. Note that since $e^v_s \in [0, C_{M_v}]$, this constraint ensures that \cref{eq:energy_constr} is satisfied.



    

\paragraph{Objective}
We can express \cref{equ:obj_main} as minimizing
\begin{small}
\begin{align*}
    \smashoperator{\sum_{v \in \calV}} K^{M_v} \! \left[ \sum_{t \in \calT} a_{v,t} \!\cdot\! E(v, t) + \smashoperator{\sum_{x_1, x_2 \in \calT \cup \calC}} m_{v,x_1,x_2} \!\cdot\! E(v, \MT(x_1, x_2)) \right] 
\end{align*}
\end{small}
where $K^{M_v}$ is $K^\text{elec}$ if $M_v \in \calM^\text{elec}$ and $K^\text{gas}$ otherwise.



\paragraph{Complexity}
The integer program contains both variables and constraints in the order of $\mathcal{O}(|\mathcal{V}|\cdot|\mathcal{T}|^2)$.\Aron{if we have space, add note about fraction of overlaps}

\subsection{Greedy Algorithm}
\label{sec:greedy}

Next, we introduce a polynomial-time greedy algorithm.
The key idea of this algorithm is to choose between assignments based on a \emph{biased cost} instead of the actual cost.

\begin{algorithm}[!h]
 \caption{$\textbf{BiasedCost}(\mathcal{A}, v, x, \alpha)$}
 \label{algo:energy_cost}

\uIf {$x \in \calT$}
{
 $\textit{cost} \leftarrow E(v,x)$
}
\Else
{
 $\textit{cost} \leftarrow 0$
}
  
  $\textit{Earlier} = \left\{ \hat{x} \in \calT \cup \calC \, \middle| \, \langle v, \hat{x} \rangle \in \calA \wedge  \hat{x}^{end} \leq x^{start} \right\}$
  

 \If{$\textit{Earlier} \neq \emptyset $}
 {
   $x_{prev} = \argmax_{\hat{x} \, \in \, \textit{Earlier}} \hat{x}^{end}$
   
    $m_{prev} \leftarrow T\left(x_{prev}^{destination}, x^{origin}\right)$
    
    $\textit{cost} \leftarrow \textit{cost}  +    E(v, m_{prev}) + \alpha \cdot \left( x^{start} - x^{end}_{prev} \right)$
 }
 
   $\textit{Later} = \left\{ \hat{x} \in \calT \cup \calC \, \middle| \, \langle v, \hat{x} \rangle \in \calA \wedge  x^{end} \leq \hat{x}^{start} \right\}$
   

 \If{$\textit{Later} \neq \emptyset $}
 {
 
 $x_{next} = \argmin_{\hat{x} \, \in \, \textit{Earlier}} \hat{x}^{start}$
 
    $m_{next} \leftarrow T\left(x^{destination}, x_{next}^{origin}\right)$
    
    $\textit{cost} \leftarrow \textit{cost} +   E(v,m_{next}) + \alpha \cdot \left( x^{end} - x^{start}_{next} \right)$
    
 }

 \KwResult{$\textit{cost}$}
\end{algorithm}




  

  
    
    
   
    
    

     


\paragraph{Biased Energy Cost}

Our greedy approach uses \cref{algo:energy_cost} to compute a \emph{biased energy cost} of assigning a bus~$v$ to a transit trip or charging slot $x$. If $x$ is a transit trip (i.e., $x \in \calT$), then the base cost of the assignment is $E(v, x)$. If $x$ is a charging slot, then the base cost is zero.
To compute the actual cost, the algorithm checks if bus $v$ is already assigned to any earlier (or later) transit trips or charging slots. 
If it is, then it factors in the cost of the moving trip $m_{prev}$ (and $m_{next}$) from the preceding (and to the following) assignment $x_{prev}$ (and $x_{next}$).
Finally, the algorithm adds a bias to the actual cost based on the waiting time between $x_{prev}^{end}$ and $x^{start}$ (if $x_{prev}$ exists) and between $x^{end}$ and $x_{next}^{end}$ (if $x_{next}$ exists).
By adding these waiting times to the cost with an appropriate factor $\alpha > 0$, we nudge the greedy selection towards increasing bus utilization and minimizing layovers.
The time complexity of this algorithm is $\mathcal{O}\left(|\mathcal{T} \cup \mathcal{C}|\right)$.
\begin{algorithm}[!h]
 \caption{$\textbf{Greedy}(\mathcal{V}, \mathcal{T}, \mathcal{C}, \alpha)$}
 \label{algo:greedy_approach}
 
 $\mathcal{A} \leftarrow \emptyset$ 


$\calE \leftarrow \left\{ \langle v, t \rangle \mapsto \textbf{BiasedCost}(\calA, v, t, \alpha) ~ | ~ v \in \calV, t \in \calT \right\} $
 
 \While {$|\mathcal{T}| > 0$ \textnormal{\textbf{and}} $\min [ \calE ] \neq \infty$}
 {  

      
     
           $\textit{MinimumCostAssignments} \leftarrow \text{argmin}(\calE)$
          
         $\langle v, t \rangle \leftarrow \textbf{first}(\textit{MinimumCostAssignments})$
         
        
        
        $\mathcal{A} \leftarrow \mathcal{A} \cup \{ \langle v,t \rangle \}$ 
        
        $\mathcal{T} \leftarrow \mathcal{T} \setminus \{t\}$
        
        $\calE, \calA \leftarrow \textbf{Update}(\calA, \mathcal{T}, \mathcal{C}, \calE, v, t, \alpha)$

        \Comment{update cost values $\calE$ and and add charging slots to the assignments as necessary}

}
 \KwResult{$\mathcal{A}$}
\end{algorithm}

\cref{algo:greedy_approach} shows our iterative greedy approach for assigning transit trips and charging slots to buses. 
The algorithm begins by computing the biased assignment cost for each pair of a bus $v$ and transit trip $t$ using 
$\textbf{BiasedCost}(\mathcal{A}, v, t, \alpha)$.
Starting with an empty set $\calA = \emptyset$,
the algorithm then iteratively adds assignments $(v, t) \in \calV \times \calT$ to the set, always choosing an assignment with the lowest biased cost $\textnormal{\textbf{BiasedCost}}(\calA, v, t, \alpha)$ (breaking ties arbitrarily).
After each iteration, the biased costs $\calE$ for the chosen vehicle $v$ are updated by $\textbf{Update}$, which also adds charging slot assignments as necessary (see \begin{revision}
\ifExtendedVersion
\cref{app:algorithms}).
\else
the online appendix~\cite{sivagnanam2020minimizing}).
\fi
\end{revision}
The algorithm terminates once all trips are assigned (or if it fails to find a solution).
The time complexity of \textbf{Update} is   $\mathcal{O}\left(|\calT|\cdot|\calV| + |\calT|\cdot|\calC|\cdot|\calX|\ln|\calX| \right)$,  where $\calX = \calT \cup \calC$. Since typically $|\calT| \gg |\calV|$, the complexity can be simplify into $\mathcal{O}\left(|\calT|\cdot|\calC|\cdot|\calX|\ln|\calX| \right)$.
Accordingly, the time complexity of the greedy algorithm is $\mathcal{O}\left(|\calT|^2\cdot|\calC|\cdot|\calX|\ln|\calX| \right)$. 

\subsection{Simulated Annealing}
\label{sec:annealing}
Finally,
we introduce a simulated annealing algorithm, which improves upon the output of the greedy algorithm  using iterative random search.
Starting from a greedy solution, the search takes significantly less time than starting from random solution. 
The key element of this algorithm is choosing a random ``neighboring'' solution in each iteration.

\begin{algorithm}[ht]
 \caption{$\textbf{RandomNeighbor}(\mathcal{A}, p_{swap})$}
 \label{algo:mutation}
     $\textit{NumberOfSwaps} \leftarrow \textbf{max}\{1, |\mathcal{A}| \cdot p_{swap}\}$
     
    \For {$1, \ldots, \textnormal{\textit{NumberOfSwaps}}$}
    {
     $v_1, v_2 \leftarrow \textbf{UniformRandom}(\mathcal{V})$
     
     $\mathcal{T}_1 \leftarrow \{{t} \in \mathcal{T} ~|~ \langle v_1,{t} \rangle \in \mathcal{A}\}$
     
    $\mathcal{T}_2 \leftarrow \{{t} \in \mathcal{T} ~|~ \langle v_2,{t} \rangle \in \mathcal{A}\}$
    
    $\textnormal{\textit{SplitTime}} \gets \textbf{UniformRandom}([s_1^{start}, s_\infty^{end}])$
     
    $\mathcal{T}_{noswap}, \mathcal{T}_{swap} = \textbf{SplitByStartTime}(\{\mathcal{T}_1, \mathcal{T}_2\}, \textnormal{\textit{SplitTime}})$
    
        

    
    \For{$t \in \mathcal{T}_{swap}$}
    { 
        \uIf{$t \in \mathcal{T}_1$}
        {
           $\mathcal{A} \leftarrow \mathcal{A} \setminus \{\langle v_1, t \rangle\} \cup  \{\langle v_2, {t} \rangle\}$
        }
        \Else 
        {
          $\mathcal{A} \leftarrow \mathcal{A} \setminus \{\langle v_2, t \rangle\} \cup \{ \langle v_1, {t} \rangle\}$
        }
        }

    }
     
     
     
     
    
       
     
    
 \KwResult{$\mathcal{A}$}
\end{algorithm}
\paragraph{Random Neighbor}
Simulated annealing uses \cref{algo:mutation} to generate a random  neighbor for a candidate solution $\mathcal{A}$. The algorithm first chooses two vehicles $v_1,v_2$ at random from $\mathcal{V}$. Next, the algorithm enumerates all the trips $\calT_1$ and $\calT_2$ that are assigned to these vehicles in solution $\calA$,
chooses a random point in time $\textit{SplitTime}$ during the day, and splits all these trips into two sets $\mathcal{T}_{noswap}$ and $\mathcal{T}_{swap}$ based on the start times of the trip and $\textit{SplitTime}$.
Finally, the algorithms swaps all the trips in $\mathcal{T}_{swap}$ between $v_1$ and $v_2$ (i.e., trips that were assigned to $v_1$ are re-assigned to $v_2$ and vice versa).
The algorithm then repeats this process from the beginning until the desired number of swap operations $|\mathcal{A}| \cdot p_{swap}$ is reached. 
The time complexity of this algorithm is $\mathcal{O}\left(|\calA|\cdot(|\calT|+|\calV|)\right)$. 


\begin{algorithm}[!h]
 \caption{$\textbf{Simulated Annealing}(\mathcal{V}, \mathcal{T}, \mathcal{C},  \alpha, k_{max}, \newline p_{start}, p_{end}, p_{swap})$}
 \label{algo:sti_mul_anneal}

$\mathcal{A} \leftarrow \textbf{Greedy}(\mathcal{V}, \mathcal{T}, \mathcal{C}, \alpha)$


$\textnormal{\textit{Solutions}} \leftarrow \{ \mathcal{A} \}$

$\tau_{start} \leftarrow \frac{-1}{\ln{p_{start}}}$

$\tau_{end} \leftarrow \frac{-1}{\ln{p_{end}}}$

$\tau_{rate} \leftarrow \left(\frac{\tau_{end}}{\tau_{start}}\right)^{\frac{1}{k_{max} -1}} $

$\tau_{k} \leftarrow \tau_{start}$

$\delta_{avg} \leftarrow 0$


 \For {$k = 1, 2, \ldots, k_{max}$}
 {  
    
    $\mathcal{A'} \leftarrow \textbf{RandomNeighbor}(\mathcal{A}, p_{swap})$

    $\delta_e \leftarrow \textbf{Cost}(\mathcal{A'}) - \textbf{Cost}(\mathcal{A})$
    
    \If{k = 1}
    {
        $\delta_{avg} \leftarrow \delta_e$
    }
    
    $\textit{AcceptProbability} \leftarrow \exp{\left(\frac{-\delta_e}{\delta_{avg} \cdot \tau_k}\right)}$

    \If{$\textnormal{\textbf{Cost}}(\calA') < \textnormal{\textbf{Cost}}(\calA)  ~\textnormal{\textbf{or}}~ \textit{AcceptProbability} > \textnormal{\textbf{UniformRandom}}([0,1])$}
    {
    
     $\mathcal{A} \leftarrow  \mathcal{A'}$
     
    $\delta_{avg} \leftarrow \delta_{avg} + \frac{\delta_e - \delta_{avg}}{|Solutions|} $
          
     $\textnormal{\textit{Solutions}} \leftarrow \textnormal{\textit{Solutions}} \cup \{\mathcal{A}\}$

    }
     
     $\tau_{k} \leftarrow \tau_{k} \cdot \tau_{rate}$   
 }
 
 $\calA^* \gets \argmin_{\calA' \in \textnormal{\textit{Solutions}}} \textbf{Cost}(\calA')$
 
 \KwResult{$\calA^*$} 
\end{algorithm}

\cref{algo:sti_mul_anneal} shows our simulated annealing approach.
\Aron{this algorithms is ``textbook'' simulated annealing, right? we should provide a reference and perhaps say this}
First, the algorithm obtains an initial solution $\mathcal{A}$ using \cref{algo:greedy_approach}. Then, it follows an iterative process. In each iteration, the algorithm obtains a random neighboring solution $\mathcal{A}'$ of the current solution $\mathcal{A}$ using \textbf{RandomNeighbor}. If the energy cost \textbf{Cost} (see \cref{equ:obj_main}) of $\mathcal{A}'$ is lower than the energy cost of $\mathcal{A}$, then the algorithm always accepts $\mathcal{A}'$ as the new solution. Otherwise, the algorithm computes the probability $\textit{AcceptProbability}$ of accepting it based on a decreasing temperature value $\tau_k$ and the cost difference between $\calA'$ and $\calA$, and then accepts $\calA'$ at random. 
The algorithm terminates after a fixed number of iterations $k_{max}$ and returns the best solution found up to that point. 
The time complexity of this algorithm is
\Aron{$\calA$ cannot be in BigO since it is not given as input}
$\mathcal{O}\left(k_{max}\cdot |\calT|\cdot(|\calT|+|\calV|)\right)$.

\sectionClearPage
\definecolor{ColorCustomGreen}{rgb}{0, 0.8, 0}
\colorlet{ColorIP}{ColorCustomGreen} 
\colorlet{ColorLegendIP}{ColorIP!50}
\colorlet{ColorSimAnn}{red}
\colorlet{ColorLegendSimAnn}{ColorSimAnn!50}
\colorlet{ColorGenAlg}{orange}
\colorlet{ColorLegendGenAlg}{ColorGenAlg!50}
\colorlet{ColorGreedy}{blue}
\colorlet{ColorLegendGreedy}{ColorGreedy!50}
\colorlet{ColorReal}{black}
\colorlet{ColorLegendReal}{ColorReal!50}

\section{Numerical Results}
\label{sec:numerical}
We evaluate our algorithms using data collected from~\begin{revision}CARTA\end{revision}. 
We will release the complete dataset as well as our implementation publicly.

\subsection{Dataset}


\paragraph{Public Transit Schedule}
We obtain the schedule of the transit agency in GTFS format,
which includes all trips, time schedules, bus stop locations, etc. 
Trips are organized into 17 bus lines (i.e., bus routes) throughout the city.
For our numerical evaluation, we consider trips served during weekdays (Monday to Friday) since these are the busiest days. 
Each weekday, the agency must serve around 850 trips using  
3 electric buses of model BYD K9S, and 50 diesel and hybrid buses of various models. 

\paragraph{Energy Use Prediction}
To estimate the energy usage of each transit and non-service trip, we use a neural network based prediction model, which we train on high-resolution historical data.
\begin{revision}CARTA\end{revision} has installed sensors on its mixed-fleet of vehicles, and it has been collecting data continuously for over a year at 1-second intervals from 3 electric, 41 diesel, and 6 hybrid buses.
To train the predictor, we select 6 months of data from 3 electric 
\begin{revision}
vehicles (BYD K9S buses)
\end{revision} and 3 diesel vehicles 
\begin{revision}
(2014 Gillig Phantom buses). We obtain 0.1 Hz timeseries data from onboard telemetry devices, which record location (GPS), odometer, battery current, voltage, and charge (for EVs), and fuel level and usage (for diesel).\end{revision}
In total, we obtain around 6.6 million datapoints for electric buses and 1.1 million datapoints for diesel buses (fuel data was recorded less frequently). 

We augment this dataset with additional features related to weather, road, and traffic conditions to improve our energy-use predictor.
We incorporate hourly  predictions of weather features, which are based on data collected using Dark Sky API \cite{darkskyapi} at 5-minute intervals. Weather features include temperature, humidity, pressure, wind speed, and precipitation. 
We include road-condition features based on a street-level map of the city obtained from OpenStreetMap.
We also include road gradients, which we compute along transit routes using an elevation map that is based on high-accuracy LiDAR data from the state government. 
Finally, 
we incorporate predictions of traffic conditions, which are  based on data obtained using HERE Maps API~\cite{heremaps}.
\begin{revision}
 Note that we present more detailed description of the dataset in
\ifExtendedVersion \cref{app:dataset}.
\else
\cite{aymandata}.
\fi
\end{revision}

In total, we use 26 different features to train neural network models for energy prediction.
%
Our neural network  has one input, two hidden, and one output layer, all using sigmoid activation.
We chose  this architecture based on its accuracy after comparing it to various other regression models. 
We train a different prediction model for each vehicle model, which we then use to predict energy use for every~trip.
\begin{revision}
For 30-minute trips, the mean absolute percentage error of the energy predictor is 8.4\% for diesel and 13.1\% for electric vehicles. For an entire day of service, the error is only 1.5\% and 2.6\%, respectively.
\end{revision}

\paragraph{Non-Service Trips}
Since non-service trips are not part of the transit schedule, we need to plan their routes and estimate their durations.
For this, we use the Google Directions API, which we query for all 2,070 possible non-service trips (i.e., for every pair of locations in the network) for each 1-hour interval of a selected weekday from 5am to 11pm. 
%
The response to each query includes an estimated duration as well as a detailed route, which we combine with our other data sources and then feed into our energy-use predictors.
 
\paragraph{Charging Rate, Energy Costs, and CO$_2$ Emissions}
Electric buses of model BYD K9S 
have a battery capacity of 270 kWh, and the charging poles of the agency can charge a BYD K9S model bus at the rate of $65$ kW/h. 
We consider 3 charging poles for our numerical evaluation.
Based on data from the transit agency, we consider electricity cost to be \$9.602 per 100 kWh and diesel cost to be \$2.05 per gallon. Finally, based on data from EPA \cite{greenhouse}, we calculate CO$_2$ emissions for diesel vehicles as 8.887 kg/gallons and for electric vehicles as 0.707 kg/kWh.

\subsection{Results}
\label{sec:results}

\begin{revision}
For all experiments, we set the length of time slots to be 1 hour.
For experiments with small problem instances, we set the wait-time factor (see \cref{algo:energy_cost}) to $\alpha = 0.00004$ for electric buses and
to $\alpha = 0.00002$ for liquid-fuel buses; swapping rate to $p_{swap} = 0.03$ (see \cref{algo:mutation}); simulated-annealing iterations, initial probability, and final probability to $k_{max} = 50,000$, $p_{start} = 0.1$, and $p_{end} = 0.07$, respectively (see \cref{algo:sti_mul_anneal}). 
For experiments with complete daily schedules, we set wait-time factor (see \cref{algo:energy_cost}) to $\alpha = 0.001$ for electric buses  and
to $\alpha = 0.0002$ for liquid-fuel buses; swapping rate to $p_{swap} = 0.05$; simulated-annealing iterations, initial probability, and final probability to $k_{max} = 50,000$, $p_{start} = 0.2$, and $p_{end} = 0.09$, respectively.
\end{revision}

We found these to be optimal configuration based on a grid search of the parameter space. 
\begin{revision}
Due to the lack of space, we include these search results in 
\ifExtendedVersion
\cref{app:numerical}.
\else
an online appendix~\cite{sivagnanam2020minimizing}.
\fi
\end{revision}

\pgfplotstableread[col sep=comma,]{data/small/times/greedy_times.csv}\greedytime
\pgfplotstableread[col sep=comma,]{data/small/times/sim_anneal_times.csv}\simannealtime
\pgfplotstableread[col sep=comma,]{data/small/times/ip_times.csv}\iptime
\begin{revision}
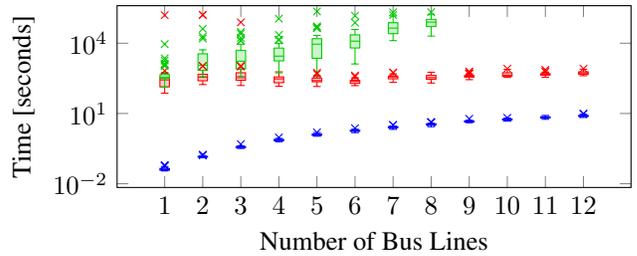
\begin{figure}[!ht]
\begin{tikzpicture}
\begin{axis}[
      boxplot/draw direction=y,
      xtick={1,2,3,4,5,6,7,8,9,10,11,12},
      width=\columnwidth,
      height = 4cm,
      bugsResolvedStyle/.style={},
      ylabel={Time [seconds]},
      xlabel={Number of Bus Lines},
      ymode=log,
      ymin=0,
      ymax=400000,
    ]
    \addplot+[boxplot={box extend=0.25, draw position=1}, ColorGreedy, solid, fill=ColorGreedy!20, mark=x] table [col sep=comma, y=line_1] {\greedytime};
\addplot+[boxplot={box extend=0.25, draw position=2}, ColorGreedy, solid, fill=ColorGreedy!20, mark=x] table [col sep=comma, y=line_2] {\greedytime};
\addplot+[boxplot={box extend=0.25, draw position=3}, ColorGreedy, solid, fill=ColorGreedy!20, mark=x] table [col sep=comma, y=line_3] {\greedytime};
\addplot+[boxplot={box extend=0.25, draw position=4}, ColorGreedy, solid, fill=ColorGreedy!20, mark=x] table [col sep=comma, y=line_4] {\greedytime};
\addplot+[boxplot={box extend=0.25, draw position=5}, ColorGreedy, solid, fill=ColorGreedy!20, mark=x] table [col sep=comma, y=line_5] {\greedytime};
\addplot+[boxplot={box extend=0.25, draw position=6}, ColorGreedy, solid, fill=ColorGreedy!20, mark=x] table [col sep=comma, y=line_6] {\greedytime};
\addplot+[boxplot={box extend=0.25, draw position=7}, ColorGreedy, solid, fill=ColorGreedy!20, mark=x] table [col sep=comma, y=line_7] {\greedytime};
\addplot+[boxplot={box extend=0.25, draw position=8}, ColorGreedy, solid, fill=ColorGreedy!20, mark=x] table [col sep=comma, y=line_8] {\greedytime};
\addplot+[boxplot={box extend=0.25, draw position=9}, ColorGreedy, solid, fill=ColorGreedy!20, mark=x] table [col sep=comma, y=line_9] {\greedytime};
\addplot+[boxplot={box extend=0.25, draw position=10}, ColorGreedy, solid, fill=ColorGreedy!20, mark=x] table [col sep=comma, y=line_10] {\greedytime};
\addplot+[boxplot={box extend=0.25, draw position=11}, ColorGreedy, solid, fill=ColorGreedy!20, mark=x] table [col sep=comma, y=line_11] {\greedytime};
\addplot+[boxplot={box extend=0.25, draw position=12}, ColorGreedy, solid, fill=ColorGreedy!20, mark=x] table [col sep=comma, y=line_12] {\greedytime};
\addplot+[boxplot={box extend=0.25, draw position=1}, ColorSimAnn, solid, fill=ColorSimAnn!20, mark=x] table [col sep=comma, y=line_1] {\simannealtime};
\addplot+[boxplot={box extend=0.25, draw position=2}, ColorSimAnn, solid, fill=ColorSimAnn!20, mark=x] table [col sep=comma, y=line_2] {\simannealtime};
\addplot+[boxplot={box extend=0.25, draw position=3}, ColorSimAnn, solid, fill=ColorSimAnn!20, mark=x] table [col sep=comma, y=line_3] {\simannealtime};
\addplot+[boxplot={box extend=0.25, draw position=4}, ColorSimAnn, solid, fill=ColorSimAnn!20, mark=x] table [col sep=comma, y=line_4] {\simannealtime};
\addplot+[boxplot={box extend=0.25, draw position=5}, ColorSimAnn, solid, fill=ColorSimAnn!20, mark=x] table [col sep=comma, y=line_5] {\simannealtime};
\addplot+[boxplot={box extend=0.25, draw position=6}, ColorSimAnn, solid, fill=ColorSimAnn!20, mark=x] table [col sep=comma, y=line_6] {\simannealtime};
\addplot+[boxplot={box extend=0.25, draw position=7}, ColorSimAnn, solid, fill=ColorSimAnn!20, mark=x] table [col sep=comma, y=line_7] {\simannealtime};
\addplot+[boxplot={box extend=0.25, draw position=8}, ColorSimAnn, solid, fill=ColorSimAnn!20, mark=x] table [col sep=comma, y=line_8] {\simannealtime};
\addplot+[boxplot={box extend=0.25, draw position=9}, ColorSimAnn, solid, fill=ColorSimAnn!20, mark=x] table [col sep=comma, y=line_9] {\simannealtime};
\addplot+[boxplot={box extend=0.25, draw position=10}, ColorSimAnn, solid, fill=ColorSimAnn!20, mark=x] table [col sep=comma, y=line_10] {\simannealtime};
\addplot+[boxplot={box extend=0.25, draw position=11}, ColorSimAnn, solid, fill=ColorSimAnn!20, mark=x] table [col sep=comma, y=line_11] {\simannealtime};
\addplot+[boxplot={box extend=0.25, draw position=12}, ColorSimAnn, solid, fill=ColorSimAnn!20, mark=x] table [col sep=comma, y=line_12] {\simannealtime};
\addplot+[boxplot={box extend=0.25, draw position=1}, ColorIP, solid, fill=ColorIP!20, mark=x] table [col sep=comma, y=line_1] {\iptime};
\addplot+[boxplot={box extend=0.25, draw position=2}, ColorIP, solid, fill=ColorIP!20, mark=x] table [col sep=comma, y=line_2] {\iptime};
\addplot+[boxplot={box extend=0.25, draw position=3}, ColorIP, solid, fill=ColorIP!20, mark=x] table [col sep=comma, y=line_3] {\iptime};
\addplot+[boxplot={box extend=0.25, draw position=4}, ColorIP, solid, fill=ColorIP!20, mark=x] table [col sep=comma, y=line_4] {\iptime};
\addplot+[boxplot={box extend=0.25, draw position=5}, ColorIP, solid, fill=ColorIP!20, mark=x] table [col sep=comma, y=line_5] {\iptime};
\addplot+[boxplot={box extend=0.25, draw position=6}, ColorIP, solid, fill=ColorIP!20, mark=x] table [col sep=comma, y=line_6] {\iptime};
\addplot+[boxplot={box extend=0.25, draw position=7}, ColorIP, solid, fill=ColorIP!20, mark=x] table [col sep=comma, y=line_7] {\iptime};
\addplot+[boxplot={box extend=0.25, draw position=8}, ColorIP, solid, fill=ColorIP!20, mark=x] table [col sep=comma, y=line_8] {\iptime};
\end{axis}
\end{tikzpicture}
\caption{Computation times for assignments using integer program (\textcolor{ColorLegendIP}{$\blacksquare$}), simulated annealing (\textcolor{ColorLegendSimAnn}{$\blacksquare$}), and greedy algorithm (\textcolor{ColorLegendGreedy}{$\blacksquare$}).
Please note the logarithmic scale on the vertical axis.}
\label{fig:time_all}
\end{figure}

\end{revision}

\pgfplotstableread[col sep=comma,]{data/small/costs/greedy_diff_costs.csv}\greedycost
\pgfplotstableread[col sep=comma,]{data/small/costs/sim_anneal_diff_costs.csv}\simannealcost

\begin{revision}

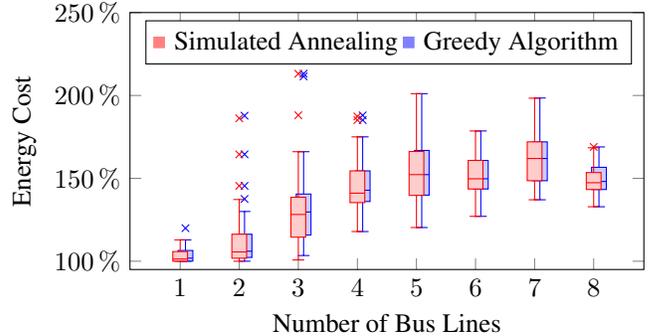
\begin{figure}
\begin{tikzpicture}
\begin{axis}[
      boxplot/draw direction=y,
      xtick={1,2,3,4,5,6,7,8},
      width=\columnwidth,
      height = 5cm,
      bugsResolvedStyle/.style={},
      ylabel={Energy Cost},
      xlabel={Number of Bus Lines},
      ymin=95,
      ymax=250,
      yticklabel=\pgfmathprintnumber{\tick}\,$\%$,
      legend pos=north west,
      legend columns=2,
      legend cell align={left},
    ]
    \addlegendimage{ColorLegendSimAnn, only marks, mark=square*};
    \addlegendentry{\,Simulated Annealing};
    \addlegendimage{ColorLegendGreedy, only marks, mark=square*}
    \addlegendentry{\,Greedy Algorithm};
    
    \addplot+[boxplot={box extend=0.25, draw position=1},ColorGreedy, rshift, solid, fill=ColorGreedy!20, mark=x] table [col sep=comma, y=line_1] {\greedycost};
\addplot+[boxplot={box extend=0.25, draw position=2},ColorGreedy, rshift, solid, fill=ColorGreedy!20, mark=x] table [col sep=comma, y=line_2] {\greedycost};
\addplot+[boxplot={box extend=0.25, draw position=3},ColorGreedy, rshift, solid, fill=ColorGreedy!20, mark=x] table [col sep=comma, y=line_3] {\greedycost};
\addplot+[boxplot={box extend=0.25, draw position=4},ColorGreedy, rshift, solid, fill=ColorGreedy!20, mark=x] table [col sep=comma, y=line_4] {\greedycost};
\addplot+[boxplot={box extend=0.25, draw position=5},ColorGreedy, rshift, solid, fill=ColorGreedy!20, mark=x] table [col sep=comma, y=line_5] {\greedycost};
\addplot+[boxplot={box extend=0.25, draw position=6},ColorGreedy, rshift, solid, fill=ColorGreedy!20, mark=x] table [col sep=comma, y=line_6] {\greedycost};
    \addplot+[boxplot={box extend=0.25, draw position=7},ColorGreedy, rshift, solid, fill=ColorGreedy!20, mark=x] table [col sep=comma, y=line_7] {\greedycost};
\addplot+[boxplot={box extend=0.25, draw position=8},ColorGreedy, rshift, solid, fill=ColorGreedy!20, mark=x] table [col sep=comma, y=line_8] {\greedycost};

    \addplot+[boxplot={box extend=0.25, draw position=1}, ColorSimAnn, solid, fill=ColorSimAnn!20, mark=x] table [col sep=comma, y=line_1] {\simannealcost};
\addplot+[boxplot={box extend=0.25, draw position=2}, ColorSimAnn, solid, fill=ColorSimAnn!20, mark=x] table [col sep=comma, y=line_2] {\simannealcost};
\addplot+[boxplot={box extend=0.25, draw position=3}, ColorSimAnn, solid, fill=ColorSimAnn!20, mark=x] table [col sep=comma, y=line_3] {\simannealcost};
\addplot+[boxplot={box extend=0.25, draw position=4}, ColorSimAnn, solid, fill=ColorSimAnn!20, mark=x] table [col sep=comma, y=line_4] {\simannealcost};
\addplot+[boxplot={box extend=0.25, draw position=5}, ColorSimAnn, solid, fill=ColorSimAnn!20, mark=x] table [col sep=comma, y=line_5] {\simannealcost};
\addplot+[boxplot={box extend=0.25, draw position=6}, ColorSimAnn, solid, fill=ColorSimAnn!20, mark=x] table [col sep=comma, y=line_6] {\simannealcost};
    \addplot+[boxplot={box extend=0.25, draw position=7}, ColorSimAnn, solid, fill=ColorSimAnn!20, mark=x] table [col sep=comma, y=line_7] {\simannealcost};
\addplot+[boxplot={box extend=0.25, draw position=8}, ColorSimAnn, solid, fill=ColorSimAnn!20, mark=x] table [col sep=comma, y=line_8] {\simannealcost};
7.07 \end{axis}
\end{tikzpicture}
\caption{Energy cost for assignments using simulated annealing (\textcolor{ColorLegendSimAnn}{$\blacksquare$}) and the greedy algorithm~(\textcolor{ColorLegendGreedy}{$\blacksquare$}) compared to optimal assignments (found using the integer program).}
	\label{fig:cost_all}
\end{figure}

\end{revision}

\paragraph{Computational Performance}

We first study how well our algorithms scale with increasing problem sizes. 
To this end, we measure the computation times of our algorithms with 1 to 12 bus lines (selected from the real bus lines), and 10 selected trips for each line. 
For each case, we evaluate the algorithms on 35 different samples and present statistical results.
For cases with 11 and 12 bus lines, we assign the entire vehicle fleet, which consists of 3 electric and 50 liquid-fuel buses.
For cases with fewer bus lines, we assume that the agency has 3 electric buses but only 5 times as many liquid-fuel buses as bus lines. 
We solve the integer program (IP) using IBM CPLEX.
\Aron{still true?} 
\Aron{but you were using only a single core for each algorithm run, right? then we should say  that you ran it on single core, not 28 cores}
We run all algorithms on a machine with a Xeon E5-2680 CPU, which has 28 cores, and 128 GB of RAM.
\cref{fig:time_all} shows the computation times for the IP, greedy, and simulated annealing.
As expected, the time to solve the IP is significantly higher and increases rapidly with the number of lines, becoming infeasible at around 10 lines.
On the the hand, the greedy and simulated annealing algorithms are orders of magnitude faster and scale well. 

\paragraph{Solution Quality}
Next, we evaluate the performance of our algorithms with respect to solution quality, that is, energy cost.
Note that we present CO$_2$ results in \begin{revision}
\ifExtendedVersion
\cref{app:numerical}.
\else
an online appendix~\cite{sivagnanam2020minimizing}.
\fi
\end{revision}We use the exact same setting as in the previous experiment (\cref{fig:time_all}).
Note that for larger instances, solving the IP is infeasible. 
\cref{fig:cost_all} shows that simulated annealing performs slightly better than greedy; however, neither perform as well as IP (which is optimal).
On the bright side, the cost ratio between IP and our heuristics \begin{revision}remains within the range of 1.5 to 1.6\end{revision} even for larger instances.




\pgfplotstableread[col sep=comma,]{data/full/costs/greedy_diff_costs.csv}\greedyco
\pgfplotstableread[col sep=comma,]{data/full/costs/sim_anneal_diff_costs.csv}\ipco

\begin{revision}
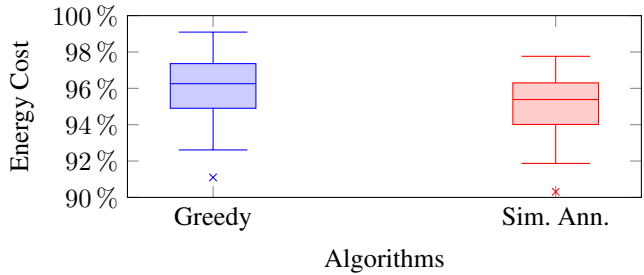
\begin{figure}
\begin{tikzpicture}
\begin{axis}[
      boxplot/draw direction=y,
      xtick={1,2},
      xticklabels={{Greedy}, {Sim. Ann.}},
      width=\columnwidth,
      height = 4cm,
      bugsResolvedStyle/.style={},
      ylabel={Energy Cost},
      xlabel={Algorithms},
      yticklabel=\pgfmathprintnumber{\tick}\,$\%$,
      ymin=90,
      ymax=100,
    ]
\addplot+[boxplot={box extend=0.25, draw position=1},ColorGreedy, solid, fill=ColorGreedy!20, mark=x] table [col sep=comma, y=greedy] {\greedyco};
\addplot+[boxplot={box extend=0.25, draw position=2}, ColorSimAnn, solid, fill=ColorSimAnn!20, mark=x] table [col sep=comma, y=sim_anneal] {\ipco};
\end{axis}
\end{tikzpicture}
\caption{Energy costs for assignments using the greedy algorithm (\textcolor{ColorLegendGreedy}{$\blacksquare$}) and simulated annealing (\textcolor{ColorLegendSimAnn}{$\blacksquare$})  for complete daily schedules,
compared to existing real-world assignments.} 
\label{fig:cost_full_sch}
\end{figure}
\end{revision}

\paragraph{Comparison to Existing Assignments}




Finally, we compute assignments for the complete daily schedule of the agency using 3 electric and 50 liquid-fuel buses using greedy and simulated annealing algorithms. In \cref{fig:cost_full_sch}, we compare greedy and simulated annealing assignments with real-world assignments for 50 different sample days based on energy costs. Our results shows that both algorithms attain lower energy costs than existing real-world assignments. On average, real-world assignments cost \$8187 with 35.58 metric tons of  CO$_2$ emission, greedy approach costs~\begin{revision}\$7863\end{revision} with \begin{revision}34.33\end{revision}~metric tons of CO$_2$ emission, and simulated annealing algorithm costs~\begin{revision}\$7788\end{revision} with \begin{revision}34.00\end{revision} metric tons of CO$_2$ emission. We were able to assign the full schedule using greedy algorithm in around 6 minutes; meanwhile, simulated annealing runs for around 8 hours (around 50,000 iterations).
Since an agency might need to find a new assignment urgently (e.g., because some buses are unavailable due to maintenance), the greedy algorithm can be a better option.

\iAron{please add here the average cost and emission values for each method over these days}

\ad{say something about whether these algorithms will have to be rerun daily because of bus breakdowns or schedule changes. hence time complexity matters. Again - project the trend to full state space of~\begin{revision}CARTA\end{revision}. num lines and num buses. What is the worst case?}

\sectionClearPage
\begin{revision}
\section{Related Work}
\label{sec:related}

Previous research efforts predict energy consumption through simulation models using spatial and temporal data (e.g.,~\citeauthor{wang2018bcharge}~\shortcite{wang2018bcharge}, \citeauthor{tian2016real}~\shortcite{tian2016real}, \citeauthor{wang2017data}~\shortcite{wang2017data}) collect GPS data, bus stop data, bus transaction data, traffic data, and electricity consumption data). Unlike the previous works, we derive realistic energy estimates using our energy predictors based on vehicle locations, traffic, elevation, and~weather data. Further, some works consider fixed energy cost and emission with respect to the miles traveled by the bus (e.g., \citeauthor{santos2016towards}~\shortcite{santos2016towards}, \citeauthor{paul2014operation}~\shortcite{paul2014operation}, \citeauthor{sassi2014vehicle}~\shortcite{sassi2014vehicle}). \citeauthor{li2014transit}~\shortcite{li2014transit} computes the distance of non-service trips using point-to-curve matching, then obtain optimal paths using Dijkstra's algorithm. But these assumptions limit applicability or performance for real-world implementation.

Researchers have applied various approaches such as genetic algorithm (\citeauthor{yang2020multi}~\shortcite{yang2020multi}, \citeauthor{sun2020optimizing}~\shortcite{sun2020optimizing}, \citeauthor{santos2016towards}~\shortcite{santos2016towards}), simulated annealing (\citeauthor{zhou2020collaborative}~\shortcite{zhou2020collaborative}), and column generation (\citeauthor{li2014transit}~\shortcite{li2014transit}) in the domain of energy-efficient bus scheduling with either liquid-fuel buses or electric buses. But only few research efforts (e.g., \citeauthor{santos2016towards}~\shortcite{santos2016towards}, \citeauthor{zhou2020collaborative}~\shortcite{zhou2020collaborative}) focused on transit networks operating mixed-fleets of buses. 

Other researchers have applied approaches such as integer programming (\citeauthor{nageshrao2017charging}~\shortcite{nageshrao2017charging}, \citeauthor{9160687}~\shortcite{9160687}, \citeauthor{picarelli2020model}~\shortcite{picarelli2020model}), Markov decision processes (\citeauthor{wang2018bcharge}~\shortcite{wang2018bcharge}), greedy algorithms (\citeauthor{jefferies2020comprehensive}~\shortcite{jefferies2020comprehensive}), genetic algorithms (\citeauthor{gao2018electric}~\shortcite{gao2018electric}, \citeauthor{chao2013optimizing}~\shortcite{chao2013optimizing}),
space-time networks (\citeauthor{olsen2020study}~\shortcite{olsen2020study}), and dynamic programming (\citeauthor{wang2020optimal}~\shortcite{wang2020optimal}) to optimally assign electric buses to charging stations.

\citeauthor{jahic2019preemptive}~\shortcite{jahic2019preemptive} use preemptive, quasi-preemptive, and non-preemptive approaches to effectively handle the load in charging stations or garages. Since charging duration occupies a reasonable portion of the routine, \citeauthor{chao2013optimizing}~\shortcite{chao2013optimizing} propose a battery replacement technique, but this approach is inefficient for transit agencies operating with a few electric buses. \citeauthor{murphey2012intelligent}~\shortcite{murphey2012intelligent} present the development of a machine learning framework for energy management optimization in an HEV, developing algorithms based on long- and short-term knowledge about the driving environment.

Some works propose solutions that reduce energy costs by changing bus schedules or routes (e.g., \citeauthor{hassold2014improving}~\shortcite{hassold2014improving}, \citeauthor{wang2018bcharge}~\shortcite{wang2018bcharge}), which can cause inconvenience to passengers, while \citeauthor{santos2016towards}~\shortcite{santos2016towards} schedule buses ensuring that service level is unchanged. \citeauthor{kliewer2006time}~\shortcite{kliewer2006time}, \citeauthor{kliewer2008line}~\shortcite{kliewer2008line}, \citeauthor{li2019mixed}~\shortcite{li2019mixed} allow a bus to serve multiple lines instead of limiting it to a single line, which can reduce energy cost; we also implement a similar approach. \citeauthor{li2019mixed}~\shortcite{li2019mixed} group the trips as origin-destination pairs and assign them to vehicles, which also reduces energy costs. \citeauthor{liao2019comparative}~\shortcite{liao2019comparative} optimize the routing of conventional and electric vehicles using separate models; in contrast, we optimize mixed fleets of vehicles using a single, integrated model.

\end{revision}

\sectionClearPage
\section{Conclusion}
\label{sec:concl}

\iAron{rewrite}
Due to the high upfront costs of EVs, many public transit agencies can only afford to operate mixed fleets of EVs, HEVs, and ICEVs. In this paper, we formulated the novel problem of minimizing operating costs and environmental impact through assigning trips and scheduling charging for  mixed fleets of public transit vehicles;
and we provided efficient greedy and simulated annealing algorithms. Based on real-world data from~\begin{revision}CARTA\end{revision}, we demonstrated that these algorithms scale well for larger instances and can provide significant savings in terms of energy costs and CO$_2$ emission. Even though our approaches perform better than existing real-world assignments, there remains a significant gap to optimal solutions (at least for smaller instances).
In future work, we will strive to improve our algorithms to close this gap and to provide further saving to transit agencies. 
We are also publicly releasing our dataset to facilitate open research in this direction.

\iAron{acknowledge that there is room for improvement, allude to future work and research directions (data will be released)}

\sectionClearPage
\subsection*{Acknowledgment}
We thank the anonymous reviewers of this paper who provides their valuable feedback and suggestions. This material is based upon work supported by the Department of Energy, Office of Energy Efficiency and Renewable Energy (EERE), under Award Number DE-EE0008467.
Disclaimer: This report was prepared as an account of work sponsored by an agency of the United States Government. Neither the United States Government nor any agency thereof, nor any of their employees, makes any warranty, express or implied, or assumes any legal liability or responsibility for the accuracy, completeness, or usefulness of any information, apparatus, product, or process disclosed, or represents that its use would not infringe privately owned rights. 
Reference herein to any specific commercial product, process, or service by trade name, trademark, manufacturer, or otherwise does not necessarily constitute or imply its endorsement, recommendation, or favoring by the United States Government or any agency thereof. 
The views and opinions of authors expressed herein do not necessarily state or reflect those of the United States Government or any agency thereof.

This work was completed in part with resources provided by the Research Computing Data Core at the University of Houston and in part using cloud research credits provided by Google. 

\sectionClearPage
\subsection*{Ethics Statement}

\todo{quick draft, needs to be rewritten (when I have time...)}

Our research did not involve any personally identifiable information.

Our objective formulation minimizes fuel and electricity usage (both measured as energy), both of which are scaled by factors ($K^\text{gas}$ and $K^\text{elec}$). In the evaluation, we set these factors to the prices paid by the agency for fuel and electricity; hence, minimizing energy usage minimizes the agency’s monetary cost. Please note that we could tweak these factors to consider other goals, e.g., minimize environmental impact by setting the factors to the environmental footprint of fuel and electric energy usage.


\iAron{ethics statement?}
\bibliography{main}

\begin{thebibliography}{31}
\providecommand{\natexlab}[1]{#1}
\providecommand{\url}[1]{\texttt{#1}}
\providecommand{\urlprefix}{URL }
\expandafter\ifx\csname urlstyle\endcsname\relax
  \providecommand{\doi}[1]{doi:\discretionary{}{}{}#1}\else
  \providecommand{\doi}{doi:\discretionary{}{}{}\begingroup
  \urlstyle{rm}\Url}\fi

\bibitem[{Ayman et~al.(2021)Ayman, Sivagnanam, Michael, Pugliese, Dubey, and
  Laszka}]{aymandata}
Ayman, A.; Sivagnanam, A.; Michael, W.; Pugliese, P.; Dubey, A.; and Laszka, A.
  2021.
\newblock Data-Driven Prediction and Optimization of Energy Use for Transit
  Fleets of Electric and ICE Vehicles.
\newblock \emph{ACM Transactions on Internet Technology} In press.

\bibitem[{Chao and Xiaohong(2013)}]{chao2013optimizing}
Chao, Z.; and Xiaohong, C. 2013.
\newblock Optimizing battery electric bus transit vehicle scheduling with
  battery exchanging: Model and case study.
\newblock \emph{Procedia-Social and Behavioral Sciences} 96: 2725--2736.

\bibitem[{{EIA}(2019)}]{eia}
{EIA}. 2019.
\newblock {U.S.} {Energy} {Information} {Administration}: Use of energy
  explained -- Energy use for transportation (2019).
\newblock
  \url{https://www.eia.gov/energyexplained/use-of-energy/transportation.php},
  Accessed: September 9th, 2020.

\bibitem[{EPA(2020{\natexlab{a}})}]{greenhouse}
EPA. 2020{\natexlab{a}}.
\newblock {Greenhouse} {Gases} {Equivalencies} {Calculator} - {Calculations}
  {and} {References}.
\newblock
  \url{https://www.epa.gov/energy/greenhouse-gases-equivalencies-calculator-calculations-and-references},
  Accessed: September 9th, 2020.

\bibitem[{EPA(2020{\natexlab{b}})}]{ghgemissions}
EPA. 2020{\natexlab{b}}.
\newblock {U.S.} {Transportation} {Sector} {Greenhouse} {Gas} {Emissions}.
\newblock \url{https://nepis.epa.gov/Exe/ZyPDF.cgi?Dockey=P100ZK4P.pdf},
  Accessed: September 9th, 2020.

\bibitem[{Gao et~al.(2018)Gao, Guo, Ren, Zhao, Ehsan, and
  Zheng}]{gao2018electric}
Gao, Y.; Guo, S.; Ren, J.; Zhao, Z.; Ehsan, A.; and Zheng, Y. 2018.
\newblock An electric bus power consumption model and optimization of charging
  scheduling concerning multi-external factors.
\newblock \emph{Energies} 11(8): 2060.

\bibitem[{Hassold and Ceder(2014)}]{hassold2014improving}
Hassold, S.; and Ceder, A. 2014.
\newblock Improving Energy Efficiency of Public Transport Bus Services by Using
  Multiple Vehicle Types.
\newblock \emph{Transportation Research Record} 2415(1): 65--71.

\bibitem[{HERE(2020)}]{heremaps}
HERE. 2020.
\newblock {HERE} {Maps} {API}.
\newblock \url{https://developer.here.com/}, Accessed: January 21st, 2020.

\bibitem[{Jahic, Eskander, and Schulz(2019)}]{jahic2019preemptive}
Jahic, A.; Eskander, M.; and Schulz, D. 2019.
\newblock Preemptive vs. non-preemptive charging schedule for large-scale
  electric bus depots.
\newblock In \emph{2019 IEEE PES Innovative Smart Grid Technologies Europe
  (ISGT-Europe)}, 1--5. IEEE.

\bibitem[{Jefferies and G{\"o}hlich(2020)}]{jefferies2020comprehensive}
Jefferies, D.; and G{\"o}hlich, D. 2020.
\newblock A Comprehensive TCO Evaluation Method for Electric Bus Systems Based
  on Discrete-Event Simulation Including Bus Scheduling and Charging
  Infrastructure Optimisation.
\newblock \emph{World Electric Vehicle Journal} 11(3): 56.

\bibitem[{Kliewer, Gintner, and Suhl(2008)}]{kliewer2008line}
Kliewer, N.; Gintner, V.; and Suhl, L. 2008.
\newblock Line change considerations within a time-space network based
  multi-depot bus scheduling model.
\newblock In \emph{Computer-aided Systems in Public Transport}, 57--70.
  Springer.

\bibitem[{Kliewer, Mellouli, and Suhl(2006)}]{kliewer2006time}
Kliewer, N.; Mellouli, T.; and Suhl, L. 2006.
\newblock A time--space network based exact optimization model for multi-depot
  bus scheduling.
\newblock \emph{European journal of operational research} 175(3): 1616--1627.

\bibitem[{Li(2014)}]{li2014transit}
Li, J.-Q. 2014.
\newblock Transit bus scheduling with limited energy.
\newblock \emph{Transportation Science} 48(4): 521--539.

\bibitem[{Li, Lo, and Xiao(2019)}]{li2019mixed}
Li, L.; Lo, H.~K.; and Xiao, F. 2019.
\newblock Mixed bus fleet scheduling under range and refueling constraints.
\newblock \emph{Transportation Research Part C: Emerging Technologies} 104:
  443--462.

\bibitem[{Liao, Liu, and Fu(2019)}]{liao2019comparative}
Liao, W.; Liu, L.; and Fu, J. 2019.
\newblock A comparative study on the routing problem of electric and fuel
  vehicles considering carbon trading.
\newblock \emph{International Journal of Environmental Research and Public
  Health} 16(17): 3120.

\bibitem[{{Lotfi} et~al.(2020){Lotfi}, {Pereira}, {Paterakis}, {Gabbar}, and
  {Catalão}}]{9160687}
{Lotfi}, M.; {Pereira}, P.; {Paterakis}, N.; {Gabbar}, H.~A.; and {Catalão},
  J. P.~S. 2020.
\newblock Optimizing Charging Infrastructures of Electric Bus Routes to
  Minimize Total Ownership Cost.
\newblock In \emph{2020 IEEE International Conference on Environment and
  Electrical Engineering and 2020 IEEE Industrial and Commercial Power Systems
  Europe (EEEIC / I CPS Europe)}, 1--6.

\bibitem[{Murphey et~al.(2012)Murphey, Park, Chen, Kuang, Masrur, and
  Phillips}]{murphey2012intelligent}
Murphey, Y.~L.; Park, J.; Chen, Z.; Kuang, M.~L.; Masrur, M.~A.; and Phillips,
  A.~M. 2012.
\newblock Intelligent hybrid vehicle power control—Part I: Machine learning
  of optimal vehicle power.
\newblock \emph{IEEE Transactions on Vehicular Technology} 61(8): 3519--3530.

\bibitem[{Nageshrao, Jacob, and Wilkins(2017)}]{nageshrao2017charging}
Nageshrao, S.~P.; Jacob, J.; and Wilkins, S. 2017.
\newblock Charging cost optimization for EV buses using neural network based
  energy predictor.
\newblock \emph{IFAC-PapersOnLine} 50(1): 5947--5952.

\bibitem[{Olsen, Kliewer, and Wolbeck(2020)}]{olsen2020study}
Olsen, N.; Kliewer, N.; and Wolbeck, L. 2020.
\newblock A study on flow decomposition methods for scheduling of electric
  buses in public transport based on aggregated time--space network models.
\newblock \emph{Central European Journal of Operations Research} 1--37.

\bibitem[{Paul and Yamada(2014)}]{paul2014operation}
Paul, T.; and Yamada, H. 2014.
\newblock Operation and charging scheduling of electric buses in a city bus
  route network.
\newblock In \emph{17th International IEEE Conference on Intelligent
  Transportation Systems (ITSC)}, 2780--2786. IEEE.

\bibitem[{Picarelli et~al.(2020)Picarelli, Rinaldi, D’Ariano, and
  Viti}]{picarelli2020model}
Picarelli, E.; Rinaldi, M.; D’Ariano, A.; and Viti, F. 2020.
\newblock Model and Solution Methods for the Mixed-Fleet Multi-Terminal Bus
  Scheduling Problem.
\newblock \emph{Transportation Research Procedia} 47: 275--282.

\bibitem[{Santos et~al.(2016)Santos, Kokkinogenis, de~Sousa, Perrotta, and
  Rossetti}]{santos2016towards}
Santos, D.; Kokkinogenis, Z.; de~Sousa, J.~F.; Perrotta, D.; and Rossetti,
  R.~J. 2016.
\newblock Towards the integration of electric buses in conventional bus fleets.
\newblock In \emph{2016 IEEE 19th International Conference on Intelligent
  Transportation Systems (ITSC)}, 88--93. IEEE.

\bibitem[{Sassi, Cherif, and Oulamara(2015)}]{sassi2014vehicle}
Sassi, O.; Cherif, W.~R.; and Oulamara, A. 2015.
\newblock Vehicle routing problem with mixed fleet of conventional and
  heterogenous electric vehicles and time dependent charging costs.
\newblock \emph{International Journal of Mathematical and Computational
  Sciences} 9(3): 171--181.

\bibitem[{Sky(2019)}]{darkskyapi}
Sky, D. 2019.
\newblock {API} Documentation.
\newblock \url{https://darksky.net/dev/docs}, Accessed: January 21st, 2020.

\bibitem[{Sun et~al.(2020)Sun, Chien, Hu, Chen, and Jiang}]{sun2020optimizing}
Sun, Q.; Chien, S.; Hu, D.; Chen, G.; and Jiang, R.-S. 2020.
\newblock Optimizing Multi-Terminal Customized Bus Service With Mixed Fleet.
\newblock \emph{IEEE Access} 8: 156456--156469.

\bibitem[{Tian et~al.(2016)Tian, Jung, Wang, Zhang, Tu, Xu, Tian, and
  Li}]{tian2016real}
Tian, Z.; Jung, T.; Wang, Y.; Zhang, F.; Tu, L.; Xu, C.; Tian, C.; and Li,
  X.-Y. 2016.
\newblock Real-time charging station recommendation system for electric-vehicle
  taxis.
\newblock \emph{IEEE Transactions on Intelligent Transportation Systems}
  17(11): 3098--3109.

\bibitem[{Wang et~al.(2018)Wang, Xie, Zhang, Liu, and Zhang}]{wang2018bcharge}
Wang, G.; Xie, X.; Zhang, F.; Liu, Y.; and Zhang, D. 2018.
\newblock {bCharge}: Data-Driven Real-Time Charging Scheduling for Large-Scale
  Electric Bus Fleets.
\newblock In \emph{2018 IEEE Real-Time Systems Symposium (RTSS)}, 45--55. IEEE.

\bibitem[{Wang, Kang, and Liu(2020)}]{wang2020optimal}
Wang, J.; Kang, L.; and Liu, Y. 2020.
\newblock Optimal scheduling for electric bus fleets based on dynamic
  programming approach by considering battery capacity fade.
\newblock \emph{Renewable and Sustainable Energy Reviews} 130: 109978.

\bibitem[{Wang et~al.(2017)Wang, Zhang, Hu, Yang, and Lee}]{wang2017data}
Wang, Y.; Zhang, D.; Hu, L.; Yang, Y.; and Lee, L.~H. 2017.
\newblock A data-driven and optimal bus scheduling model with time-dependent
  traffic and demand.
\newblock \emph{IEEE Transactions on Intelligent Transportation Systems} 18(9):
  2443--2452.

\bibitem[{Yang and Liu(2020)}]{yang2020multi}
Yang, X.; and Liu, L. 2020.
\newblock A Multi-Objective Bus Rapid Transit Energy Saving Dispatching
  Optimization Considering Multiple Types of Vehicles.
\newblock \emph{IEEE Access} 8: 79459--79471.

\bibitem[{Zhou et~al.(2020)Zhou, Xie, Zhao, and Lu}]{zhou2020collaborative}
Zhou, G.-J.; Xie, D.-F.; Zhao, X.-M.; and Lu, C. 2020.
\newblock Collaborative Optimization of Vehicle and Charging Scheduling for a
  Bus Fleet Mixed With Electric and Traditional Buses.
\newblock \emph{IEEE Access} 8: 8056--8072.

\end{thebibliography}

\ifExtendedVersion
\clearpage
\appendix

\section{Dataset Characterization}
\label{app:dataset}
For collecting data from CARTA's fleet of vehicles, we partner with ViriCiti, which offers sensor devices and an online platform to support transit operators with real-time insight into their fleets. ViriCiti has installed sensors on CARTA's mixed-fleet of 3 electric, 41 diesel, and 6 hybrid buses, and it has been collecting data continuously at 1-second (or shorter) intervals since installation.
To train the energy usage predictor, we use 6 months of data from 3 electric vehicles (BYD K9S buses) and 3 diesel vehicles (2014 Gillig Phantom buses). For each vehicle, we obtain time-series data from ViriCiti, which includes a series of timestamps and location (GPS), odometer, battery current, voltage, and charge (for EVs), and fuel level and usage (for diesel). We divide the time-series into disjoint samples that correspond to contiguous sequences driven on the same road. 

We augment the samples with additional features associated with weather, road, and traffic conditions to improve the energy use prediction. We incorporate hourly weather predictions for features such as temperature, humidity, pressure, wind speed, and precipitation; from Dark SkyAPI at 5-minute intervals. We also include road gradients, which we compute along transit routes using an elevation map based on high-accuracy LiDAR data from the state government. For incorporating traffic features, we use traffic data collected at 1-minute intervals using the HERE API, which provides speed recordings for segments of major roads of the city. A more detailed description of our dataset and energy prediction models can be found in \cite{aymandata}.

\section{Other Algorithms}
\label{app:algorithms}

Here, we present the algorithms that are invoked by the greedy algorithm as subroutines.

\begin{algorithm}[!h]
 \caption{$\textbf{Feasible}(\mathcal{A}, v, x)$}
 \label{algo:feasible}
 
 $\textit{TimeFeasible} \leftarrow \text{False}$

$\textit{EnergyFeasible} \leftarrow \text{False}$

 $\calX_\textit{assigned} \leftarrow \left\{ \hat{x} \in \calT \cup \calC \,\middle|\, (\hat{v}, \hat{x}) \in \calA \land \hat{v} \in \calV \right\} $
 
 \If {$x \notin \calX_\textit{assigned}$}
 {
 
  $\textit{TimeFeasible} \leftarrow \textbf{TimeFeasible}(\mathcal{A}, v, x)$

    $\textit{EnergyFeasible} \leftarrow \text{True}$

  $\calX_{v} \leftarrow \left\{ \hat{x} \in \calT \cup \calC \,\middle|\, (v, \hat{x}) \in \calA \right\} $

  \If{$\textit{TimeFeasible} \land M_v \in \calM^\text{elec} \land \calX_{v} \ne \emptyset $}
{
  $\textit{EnergyFeasible} \leftarrow
  \textbf{EnergyFeasible}(\mathcal{A}, v, x)$

}

 }

\KwResult{\textit{TimeFeasible}, \textit{EnergyFeasible}}
\end{algorithm}

\paragraph{Feasibility}
\cref{algo:feasible} checks the feasibility of assigning the trip or charging slot $x$ to bus $v$, without violating \cref{equ:constraint_1a,equ:constraint_1b,equ:constraint_1c,equ:constraint_1d,eq:energy_constr} (i.e., time and energy constraints). 
The algorithm first checks whether bus $v$ can be assigned to $x$ without violating  \cref{equ:constraint_1a,equ:constraint_1b,equ:constraint_1c,equ:constraint_1d} 
(\textbf{TimeFeasible}($\calA, v, x$)). 
If bus $v$ is a liquid-fuel vehicle, then the algorithm skips the remaining steps. If bus $v$ is an electric vehicle, then the algorithm checks whether bus $v$ has enough energy to be assigned to trip or charging slot $x$ without violating \cref{eq:energy_constr}
(\textbf{EnergyFeasible}($\calA, v, x$)). The time complexity of the algorithm $\textbf{Feasible}$ is $\mathcal{O}\left(|\calX|\ln|\calX|\right)$ where $\calX = \calT \cup \calC$. 
Note that we omit the pseudocode of \textbf{TimeFeasible} and \textbf{EnergyFeasible} since their implementation is straightforward based on \cref{equ:constraint_1a,equ:constraint_1b,equ:constraint_1c,equ:constraint_1d,eq:energy_constr}.

\begin{algorithm}[!h]
 \caption{$\textbf{AssignCharging}(\calA, \calC, v, t_\text{next}, \alpha)$}
 \label{algo:assign_charging}
 
  $c^* \leftarrow $ Nil
 
 $\mathcal{X}_\text{prev} = \left\{ x \in \calT \cup \calC \,\middle|\, \langle v, {x}\rangle \in \calA \wedge {x}^\text{end} \leq t^\text{start}_\text{next} \right\}$

 $x_\text{prev} \leftarrow \argmax_{{x} \in \mathcal{X}_\text{prev}} x^\text{end} $
 
 $\calC_\text{between} \leftarrow \big\{ \hat{c} \in \calC \,\big|\,$ $x_\text{prev}^\text{end} \leq  \hat{c}^\text{start}
 \wedge    \hat{c}^\text{end} \leq t_\text{next}^\text{start} \wedge (\hat{v}, \hat{c}) \notin \calA \wedge \hat{v} \in \calV \big\} $

\If {$\calC_\text{between} \neq \emptyset$}
{


$\textit{MinimumCost} \gets \infty$

\For {$c \in \calC_\text{between}$}
 {

     $\textit{Cost} \leftarrow \infty$

         $\textit{TimeFeasible}, \textit{EnergyFeasible} \leftarrow \textbf{Feasible}(\mathcal{A}, v, c)$

\If {$\textit{TimeFeasible} \land \textit{EnergyFeasible}$}
{

        $\textit{Cost} \leftarrow \textbf{BiasedCost}(\mathcal{A},c,t, \alpha)$

       }
       
       \If{$\textit{Cost} < \textit{MinimumCost}$}{
         $c^* \gets c$
         
         $\textit{MinimumCost} \gets \textit{Cost}$
       }
       
 }

     \If{$\textit{MinimumCost} < \infty$}
     {
 
        $\mathcal{A} \leftarrow \mathcal{A} \cup \{\langle v, c^*\rangle\}$
    
     }
     
}

 \KwResult{$\mathcal{A}, c^*$}
\end{algorithm}

\paragraph{AssignCharging}
\cref{algo:assign_charging} first identifies a set of charging slots ($\calC_\text{between}$) to which the electric bus $v$ could be assigned (disregarding the duration of non-service trips).
If the algorithm identifies a non-empty set of possible charging slots, then for each possible charging slot $c \in \calC_\text{between}$, the algorithm computes the feasibility (\textbf{Feasible}) and biased cost (\textbf{BiasedCost}) of the non-service trips to move bus $v$ from the destination of its previous assignment $x_\text{prev}$ to charging slot $c$ and then from charging slot $c$ to the origin of its next assignment $t_\text{next}$. 
Thereafter, the algorithm chooses the charging slot $c$ that attains minimum biased energy cost (breaking ties in favor of earlier charging slots) and assigns charging slot $c$ to bus $v$.
If the algorithm fails to find a feasible charging slot, then the assignments are unchanged, and the algorithm returns a special Nil value instead of a charging slot $c$.
The time complexity of the algorithm is $\mathcal{O}\left(|\calC|\cdot|\calX|\ln|\calX|\right)$.

\begin{algorithm}[!h]
 \caption{$\textbf{Update}(\mathcal{A}, \mathcal{T}, \mathcal{C}, \calE, v, t, \alpha)$}
 \label{algo:update_cost_matrix}

$\textit{AssignedCharging} \leftarrow \text{False}$

\For {$\hat{t} \in \mathcal{T}$}
 {
 
     $\textit{cost} \leftarrow \infty$

         $\textit{TimeFeasible}, \textit{EnergyFeasible} \leftarrow \textbf{Feasible}(\mathcal{A}, v, \hat{t})$

\uIf {$\textit{TimeFeasible} \land \textit{EnergyFeasible}$}
{

        $\textit{cost} \leftarrow \textbf{BiasedCost}(\mathcal{A},v,\hat{t},\alpha)$

       }
   \ElseIf {$\textit{TimeFeasible} \land \lnot\textit{EnergyFeasible} \land M_v \in \calM^\text{elec}$}
   {
   
    $\mathcal{A}, \text{c} \leftarrow
   \textbf{AssignCharging}(\mathcal{A}, \mathcal{C}, v, \hat{t}, \alpha)$
   
   \If{$c \neq \text{Nil}$}{

      $\textit{TimeFeasible}, \textit{EnergyFeasible} \leftarrow \textbf{Feasible}(\mathcal{A}, v, \hat{t})$
      
      \uIf {$\textit{TimeFeasible} \land \textit{EnergyFeasible}$}
     {
         $\textit{AssignedCharging} \leftarrow \text{True}$
    
            \textbf{break}
     }\Else
     {
          $\calA \leftarrow \calA \setminus \{\langle v, c\rangle\}$

     }
   }

   }
   
        $\calE[v][\hat{t}] \leftarrow \textit{cost}$
 }

\If{$\textit{AssignedCharging}$}
{
\For {$\hat{t} \in \mathcal{T}$}
 {
 
 $\textit{cost} \leftarrow \infty$

         $\textit{TimeFeasible}, \textit{EnergyFeasible} \leftarrow \textbf{Feasible}(\mathcal{A}, v, \hat{t})$

\If {$\textit{TimeFeasible} \land \textit{EnergyFeasible}$}
{

        $\textit{cost} \leftarrow \textbf{BiasedCost}(\mathcal{A},v,\hat{t},\alpha)$

       }
         $\calE[v][\hat{t}] \leftarrow \textit{cost}$

 }

}
 \KwResult{$\calE, \calA$}
\end{algorithm}

\paragraph{Update} \cref{algo:update_cost_matrix} computes the energy costs of assigning bus $v$ to serve each one of the remaining service trips $t \in \calT$ and returns an updated matrix of energy costs $\calE$. 
If bus $v$ is an electric vehicle and unable serve any one of the remaining service trips due to low battery charge, then the algorithm tries to assign bus $v$ to a feasible charging slot $c$
\Aron{this does not make sense, why would you recompute for other trips, which could be served without recharging?}
%
and recomputes the energy costs for all trips. The time complexity of the algorithm is $\mathcal{O}\left(|\calT|\cdot|\calC|\cdot|\calX|\ln|\calX| \right)$.

\newpage
\section{Computational Complexity of Assignment and Scheduling Problem}
\label{app:complexity}

Here, we show that the problem of optimally assigning vehicles to transit trips and scheduling their charging (i.e., assigning to charging slots) is computationally hard.
First, we formulate a decision version of our optimization problem.

\begin{definition}[Transit Assignment and Scheduling Problem] 
Given 
\begin{itemize}
\item a set of liquid-fuel vehicle models $\calM^{\text{gas}}$,  
\item a set of electric vehicle models $\calM^{\text{elec}}$ with a battery capacity $C_m$ for each model $m \in \calM^{\text{elec}}$, 
\item a set of vehicles $\calV$ with a vehicle model $M_v \in \calM^{\text{gas}} \cup \calM^{\text{elec}}$ for each vehicle $v \in \calV$,
\item a set of locations $\calL$,
\item a set of transit trips $\calT$ with an origin $t^{\text{origin}} \in \calL$, destination $t^{\text{destination}} \in \calL$, start time $t^{\text{start}}$, and end time $t^{\text{end}}$ for each trip $t \in \calT$, 
\item a set of charging poles $\mathcal{CP}$ with a location $cp^{\text{location}}$ for each pole $cp \in \mathcal{CP}$,
\item a set of time slots $\calS$ with a start time $s^\text{start}$ and end time $s^\text{end}$ for each slot $s \in \calS$,
\item a charging performance $P(cp, m)$ for each charging pole $cp \in \mathcal{CP}$ and vehicle model $m \in \calM^{\text{elec}}$,
\item a duration $D(l_1, l_2)$ for each non-service trip $T(l_1, l_2)$ where $l_1, l_2 \in \calL$,
\item an energy usage value $E(v, t)$ for each vehicle $v \in \calV$ and for each transit trip $t \in \calT$ and each non-service trip $t = T(l_1, l_2)$ where $l_1, l_2 \in \calL$,
\item unit cost values $K^\text{gas}$ and $K^\text{elec}$ for liquid-fuel and electric energy usage,
\item and a threshold cost value $\textbf{Cost}^*$,
\end{itemize}
determine if there exists a set of assignments $\calA$ that satisfies Equations~\eqref{equ:constraint_1a} to~\eqref{eq:energy_constr} and $\textbf{Cost}(\calA) \leq \textbf{Cost}^*$, where $\textbf{Cost}(\calA)$ is the cost of assignments $\calA$ as defined in \cref{equ:obj_main}.
\end{definition}

We show that the above decision problem is NP-hard using a reduction from a well-known NP-hard problem, the 0-1 Knapsack Problem, which is defined as follows.

\begin{definition}[0-1 Knapsack Problem (Decision Version)]
Given a set of $N$ items, numbered from $1$ to $N$, with a value $b_i$ and weight $w_i$ for each item $i \in \{1, 2, \ldots, N\}$, a weight capacity $W$, and a threshold value $B^*$,
determine if there exists a subset of items $A \subseteq \{1, 2, \ldots, N\}$ satisfying $\sum_{i \in A} w_i \leq W$ (i.e., sum weight of items does not exceed the weight capacity) and $\sum_{i \in A} b_i \geq B^*$ (i.e., sum value of items reaches the threshold).
\end{definition}

Next, we formulate our computational complexity result.

\begin{theorem}
The Transit Assignment and Scheduling Problem is NP-hard.
\end{theorem}

\begin{proof}
We show that the Transit Assignment and Scheduling Problem (TASP) is at least as hard as the 0-1 Knapsack Problem (0-1KP).
Given an instance $(N, \langle b_i, w_i \rangle_{i\in\{1,\ldots,N\}}, W, B^*)$ of 0-1KP, we construct an instance of TASP as follows:
\begin{itemize}
    \item let there be a single liquid-fuel vehicle model, denoted $m^\text{gas}$ (i.e., $\calM^\text{gas} = \{m^\text{gas}\}$);
    \item let there be a single electric vehicle model, denoted $m^\text{elec}$ (i.e., $\calM^\text{elec} = \{m^\text{elec}\}$), with battery capacity $C_{m^\text{elec}} = W$;
    \item let there be two vehicles, denoted $v^\text{gas}$ and $v^\text{elec}$ (i.e., $\calV = \{ v^\text{gas}, v^\text{elect} \}$), with vehicle models $M_{v^\text{gas}} = m^\text{gas}$ and $M_{v^\text{elec}} = m^\text{elec}$;
    \item let there be two locations, denoted $l^\text{charge}$ and $l^\text{serve}$ (i.e., $\calL = \{ l^\text{charge}, l^\text{serve} \}$);
    \item let there be $N$ transit trips, denoted $t_1, t_2, \ldots, t_N$ with the same origins and destinations $t_i^\text{origin} = t_i^\text{destination} = l^\text{serve}$ and with start and end times $t_i^\text{start} = N + i$ and $t_i^\text{end} = N + i + 1$;
    \item let there be a single charging pole, denoted $cp$ (i.e., $\mathcal{CP} = \{cp\}$), with location $cp^\text{location} = l^\text{charge}$;
    \item let there be $2N + 1$ time slots, denoted $s_1, s_2, \ldots, s_{2N+1}$, with start and end times $s_i^\text{start} = i - 1$ and $s_i^\text{end} = i$;
    \item let the charging performance be $P(cp, m^\text{elec}) = W$;
    \item let the duration of both non-service trips be $D(l^\text{charge}, l^\text{serve}) = D(l^\text{serve}, l^\text{charge}) = N$;
    \item for the liquid-fuel vehicle $v^\text{gas}$, let the energy usage value of transit trip $t_i$ be $E(v^\text{gas}, t_i) = w_i + b_i$, and let the energy usage value of the non-service trips $T(l^\text{serve}, l^\text{charge})$ and $T(l^\text{charge}, l^\text{serve})$ both be $E(v^\text{gas}, T) = 0$;
    \item for the electric vehicle $v^\text{elec}$, let the energy usage value of transit trip $t_i$ be $E(v^\text{elec}, t_i) = w_i$, and let the energy usage value of the non-service trips $T(l^\text{serve}, l^\text{charge})$ and $T(l^\text{charge}, l^\text{serve})$ both be $E(v^\text{elec}, T) = 0$;
    \item let the unit energy costs be $K^\text{gas} = K^\text{elec} = 1$;
    \item let the threshold cost value be $\textbf{Cost}^* = \left[ \sum_{i \in \{1, \ldots, N\}} b_i + w_i \right] - B^*$.
\end{itemize}
Clearly, the above reduction can be performed in a polynomial number of steps. It remains to show that the constructed instance of TASP has a solution if and only if the 0-1KP instance has a solution.

First, suppose that the 0-1KP instance has a solution $A$.
Then, we show that there exists a feasible set of assignments $\calA$ that is a solution to the TASP instance.
Let $\calA = \{\langle v^\text{elec}, (cp, s_1) \rangle\} \cup \bigcup_{i \in A} \{\langle v^\text{elec}, t_i \rangle\} \cup \bigcup_{i \not\in A} \{\langle v^\text{gas}, t_i \rangle\}$.
In other words, charge the electric vehicle in the first time slot, then assign it to all the trips that correspond to items in $A$, and assign the liquid-fuel vehicle to all other trips.
This assignment clearly satisfies Equations \eqref{equ:constraint_1a} to \eqref{equ:constraint_1d} since all transit trips share the same origin and destination, and the only non-service trip $T(l^\text{charge}, l^\text{gas})$ is taken by the electric vehicle $v^\text{elec}$, which has enough time between the end of charging at time $s_1^\text{end} = 1$ and the beginning of the first trip $t_1^\text{start} = N + 1$ for the non-service trip duration $D(l^\text{charge}, l^\text{serve}) = N$.
Further, the assignment also satisfies Equation \eqref{eq:energy_constr} since the electric vehicle $v^\text{elec}$ is fully charged to $P(cp, m^\text{elec}) = W$ in the first time slot, and the transit trips that it serves use 
\begin{align}
\sum_{\langle v^\text{elec}, t_i\rangle\in\calA} E(v^\text{elec}, t_i) 
&= \sum_{\langle v^\text{elec}, t_i\rangle\in\calA} w_i \\
&= \sum_{i \in A} w_i \\
&\leq W .  
\end{align}
Finally, the cost $\textbf{Cost}(\calA)$ of this assignment $\calA$ is at most $\textbf{Cost}^*$ since
\begin{align}
    \textbf{Cost}(\calA) &= \sum_{\mathclap{\langle v^\text{gas}, t_i\rangle \in \calA}} K^\text{gas} \cdot E(v^\text{gas}, t_i) + \sum_{\mathclap{\langle v^\text{elec}, t_i\rangle \in \calA}} K^\text{elec} \cdot E(v^\text{elec}, t_i) \label{equ:ComplexityProof1} \\
    &= \sum_{\mathclap{\langle v^\text{gas}, t_i\rangle \in \calA}}  E(v^\text{gas}, t_i) + \sum_{\mathclap{\langle v^\text{elec}, t_i\rangle \in \calA}}  E(v^\text{elec}, t_i) \\
    &= \sum_{i \not\in A}  E(v^\text{gas}, t_i) + \sum_{i \in A}  E(v^\text{elec}, t_i) \\
    &= \left[ \sum_{i \not\in A}  w_i + b_i \right] + \left[ \sum_{i \in A}  w_i \right] \\
    &= \left[ \sum_{i \in \{1, \ldots, N\}}  w_i + b_i \right] - \underbrace{\left[ \sum_{i \in A}  b_i \right]}_{\geq B^*} \\
    &\leq \left[ \sum_{i \in \{1, \ldots, N\}}  w_i + b_i \right] - B* \\
    &= \textbf{Cost}^* . \label{equ:ComplexityProof2}
\end{align}
Therefore, if the 0-1KP instance has a solution, then so does the constructed instance of TASP.

\vspace{0.5em}
Second, suppose that the constructed instance of TASP has a solution $\calA$.
Then, we show that there exists a subset of items $A$ that is a solution to the 0-1KP instance.
Let $A = \left\{ i \middle| \langle v^\text{elec}, t_i \rangle \in \calA \right\}$.
In other words, select the items that correspond to transit trips that are assigned to the electric vehicle $v^\text{elect}$ by solution $\calA$.
Then, the sum value $\sum_{i\in A} b_i$ of the set $A$ reaches the threshold value $B^*$ since
\begin{tiny}
\allowdisplaybreaks
\begin{align}
    \sum_{i\in A} b_i &= \left[ \sum_{i \in \{1, \ldots N\}} w_i + b_i \right] - \left[ \sum_{i \in \{1, \ldots N\}} w_i + b_i \right] + \sum_{i \in A} b_i  \\
    &= \left[ \sum_{i \in \{1, \ldots N\}} w_i + b_i \right] - \left[ \sum_{i \not\in A} w_i + b_i \right] - \sum_{i \in A} w_i  \\
    &= \left[ \sum_{i \in \{1, \ldots N\}} w_i + b_i \right] -  \sum_{i \not\in A} E(v^\text{gas}, t_i)  -  \sum_{i \in A} E(v^\text{elec}, t_i) \\
    &= \left[ \sum_{i \in \{1, \ldots N\}} w_i + b_i \right] -  \sum_{i \not\in A} K^\text{gas} \!\cdot\! E(v^\text{gas}, t_i)  -  \sum_{i \in A} K^\text{elec} \!\cdot\! E(v^\text{elec}, t_i) \\
    &= \left[ \sum_{i \in \{1, \ldots N\}} w_i + b_i \right] \underbrace{- ~~~~\sum_{\mathclap{\langle v^\text{gas}, t_i\rangle \in \calA}} K^\text{gas} \!\cdot\! E(v^\text{gas}, t_i)  -  \sum_{\mathclap{\langle v^\text{elec}, t_i\rangle \in \calA}} K^\text{elec} \!\cdot\! E(v^\text{elec}, t_i)}_{= -\textbf{Cost}(\calA)} \\
    &= \left[ \sum_{i \in \{1, \ldots N\}} w_i + b_i \right] - \underbrace{\textbf{Cost}(\calA)}_{\leq \textbf{Cost}^*} \\
    &\geq \left[ \sum_{i \in \{1, \ldots N\}} w_i + b_i \right] - \textbf{Cost}^* \\
    &= \left[ \sum_{i \in \{1, \ldots N\}} w_i + b_i \right] -
    \left[ \sum_{i \in \{1, \ldots, N\}} b_i + w_i \right] + B^* \\
    &= B^* .
\end{align}
\end{tiny}
It remains to show that the sum weight $\sum_{i \in A} w_i$ of the set $A$ does not exceed the weight capacity $W$.
Due to the duration of the non-service trip between $l^\text{charge}$ and $l^\text{serve}$, the feasible set of assignments $\calA$ cannot charge the electric vehicle $v^\text{elec}$ between two transit trip.
Thus, the total energy usage of the electric vehicle cannot exceed its battery capacity $W$.
We can use this to prove that set $A$ does not exceed the weight capacity as
\begin{align}
    \sum_{i \in A} w_i &= \sum_{i \in A} E(v^\text{elec}, t_i) \\
    &= \sum_{\langle v^\text{elec}, t_i\rangle \in \calA} E(v^\text{elec}, t_i) \\
    &\leq W .
\end{align}
Therefore, if the constructed TASP instance has a solution, then so does the 0-1KP instance, which concludes our proof.
\end{proof}

\section{Further Numerical Results}
\label{app:numerical}

\paragraph{CO$_2$ Emission}

In \cref{sec:numerical}, we evaluated the proposed algorithms based on running time and solution quality, which we measured as the energy costs incurred by the transit agency.
Here, we complement those results by evaluating the algorithms based on environmental impact, specifically, CO$_2$ emissions.

\pgfplotstableread[col sep=comma,]{data/small/co2_emission/greedy_cons.csv}\greedycons
\pgfplotstableread[col sep=comma,]{data/small/co2_emission/sim_anneal_cons.csv}\simannealcons
\pgfplotstableread[col sep=comma,]{data/small/co2_emission/ip_cons.csv}\ipcons
\pgfplotstableread[col sep=comma,]{data/small/co2_emission/gen_alg_cons.csv}\gacons

\begin{revision}
\begin{figure}[h!]
\begin{tikzpicture}
\begin{axis}[
      boxplot/draw direction=y,
      xtick={1,2,3,4,5,6,7,8,9,10,11,12},
      width=\columnwidth,
      height = 6cm,
      bugsResolvedStyle/.style={},
      ylabel={CO$_{2}$ Emission [kg]},
      xlabel={Number of Bus Lines},
      ymin=0,
      ymax=7000,
            legend pos=north west,
      legend cell align={left},
    ]
    \addlegendimage{ColorLegendIP, only marks, mark=square*};
    \addlegendentry{\,Integer Program};
    \addlegendimage{ColorLegendSimAnn, only marks, mark=square*};
    \addlegendentry{\,Simulated Anneal.};
    \addlegendimage{ColorLegendGreedy, only marks, mark=square*}
    \addlegendentry{\,Greedy Algorithm};
    \addlegendimage{ColorLegendGenAlg, only marks, mark=square*}
    \addlegendentry{\,Genetic Algorithm};
    
    \addplot+[boxplot={box extend=0.25, draw position=1},ColorGreedy, rshift, solid, fill=ColorGreedy!20, mark=x] table [col sep=comma, y=line_1] {\greedycons};
\addplot+[boxplot={box extend=0.25, draw position=2},ColorGreedy, rshift, solid, fill=ColorGreedy!20, mark=x] table [col sep=comma, y=line_2] {\greedycons};
\addplot+[boxplot={box extend=0.25, draw position=3},ColorGreedy, rshift, solid, fill=ColorGreedy!20, mark=x] table [col sep=comma, y=line_3] {\greedycons};
\addplot+[boxplot={box extend=0.25, draw position=4},ColorGreedy, rshift, solid, fill=ColorGreedy!20, mark=x] table [col sep=comma, y=line_4] {\greedycons};
\addplot+[boxplot={box extend=0.25, draw position=5},ColorGreedy, rshift, solid, fill=ColorGreedy!20, mark=x] table [col sep=comma, y=line_5] {\greedycons};
\addplot+[boxplot={box extend=0.25, draw position=6},ColorGreedy, rshift, solid, fill=ColorGreedy!20, mark=x] table [col sep=comma, y=line_6] {\greedycons};
    \addplot+[boxplot={box extend=0.25, draw position=7},ColorGreedy, rshift, solid, fill=ColorGreedy!20, mark=x] table [col sep=comma, y=line_7] {\greedycons};
\addplot+[boxplot={box extend=0.25, draw position=8},ColorGreedy, rshift, solid, fill=ColorGreedy!20, mark=x] table [col sep=comma, y=line_8] {\greedycons};
\addplot+[boxplot={box extend=0.25, draw position=9},ColorGreedy, rshift, solid, fill=ColorGreedy!20, mark=x] table [col sep=comma, y=line_9] {\greedycons};
\addplot+[boxplot={box extend=0.25, draw position=10},ColorGreedy, rshift, solid, fill=ColorGreedy!20, mark=x] table [col sep=comma, y=line_10] {\greedycons};
\addplot+[boxplot={box extend=0.25, draw position=11},ColorGreedy, rshift, solid, fill=ColorGreedy!20, mark=x] table [col sep=comma, y=line_11] {\greedycons};
\addplot+[boxplot={box extend=0.25, draw position=12},ColorGreedy, rshift, solid, fill=ColorGreedy!20, mark=x] table [col sep=comma, y=line_12] {\greedycons};
    \addplot+[boxplot={box extend=0.25, draw position=1}, ColorSimAnn, solid, fill=ColorSimAnn!20, mark=x] table [col sep=comma, y=line_1] {\simannealcons};
\addplot+[boxplot={box extend=0.25, draw position=2}, ColorSimAnn, solid, fill=ColorSimAnn!20, mark=x] table [col sep=comma, y=line_2] {\simannealcons};
\addplot+[boxplot={box extend=0.25, draw position=3}, ColorSimAnn, solid, fill=ColorSimAnn!20, mark=x] table [col sep=comma, y=line_3] {\simannealcons};
\addplot+[boxplot={box extend=0.25, draw position=4}, ColorSimAnn, solid, fill=ColorSimAnn!20, mark=x] table [col sep=comma, y=line_4] {\simannealcons};
\addplot+[boxplot={box extend=0.25, draw position=5}, ColorSimAnn, solid, fill=ColorSimAnn!20, mark=x] table [col sep=comma, y=line_5] {\simannealcons};
\addplot+[boxplot={box extend=0.25, draw position=6}, ColorSimAnn, solid, fill=ColorSimAnn!20, mark=x] table [col sep=comma, y=line_6] {\simannealcons};
    \addplot+[boxplot={box extend=0.25, draw position=7}, ColorSimAnn, solid, fill=ColorSimAnn!20, mark=x] table [col sep=comma, y=line_7] {\simannealcons};
\addplot+[boxplot={box extend=0.25, draw position=8}, ColorSimAnn, solid, fill=ColorSimAnn!20, mark=x] table [col sep=comma, y=line_8] {\simannealcons};
\addplot+[boxplot={box extend=0.25, draw position=9}, ColorSimAnn, solid, fill=ColorSimAnn!20, mark=x] table [col sep=comma, y=line_9] {\simannealcons};
\addplot+[boxplot={box extend=0.25, draw position=10}, ColorSimAnn, solid, fill=ColorSimAnn!20, mark=x] table [col sep=comma, y=line_10] {\simannealcons};
\addplot+[boxplot={box extend=0.25, draw position=11}, ColorSimAnn, solid, fill=ColorSimAnn!20, mark=x] table [col sep=comma, y=line_11] {\simannealcons};
\addplot+[boxplot={box extend=0.25, draw position=12}, ColorSimAnn, solid, fill=ColorSimAnn!20, mark=x] table [col sep=comma, y=line_12] {\simannealcons};
\addplot+[boxplot={box extend=0.25, draw position=1}, ColorIP, lshift, solid, fill=ColorIP!20, mark=x] table [col sep=comma, y=line_1] {\ipcons};
\addplot+[boxplot={box extend=0.25, draw position=2}, ColorIP, lshift, solid, fill=ColorIP!20, mark=x] table [col sep=comma, y=line_2] {\ipcons};
\addplot+[boxplot={box extend=0.25, draw position=3}, ColorIP, lshift, solid, fill=ColorIP!20, mark=x] table [col sep=comma, y=line_3] {\ipcons};
\addplot+[boxplot={box extend=0.25, draw position=4}, ColorIP, lshift, solid, fill=ColorIP!20, mark=x] table [col sep=comma, y=line_4] {\ipcons};
\addplot+[boxplot={box extend=0.25, draw position=5}, ColorIP, lshift, solid, fill=ColorIP!20, mark=x] table [col sep=comma, y=line_5] {\ipcons};
\addplot+[boxplot={box extend=0.25, draw position=6}, ColorIP, lshift, solid, fill=ColorIP!20, mark=x] table [col sep=comma, y=line_6] {\ipcons};
\addplot+[boxplot={box extend=0.25, draw position=7}, ColorIP, lshift, solid, fill=ColorIP!20, mark=x] table [col sep=comma, y=line_7] {\ipcons};
\addplot+[boxplot={box extend=0.25, draw position=8}, ColorIP, lshift, solid, fill=ColorIP!20, mark=x] table [col sep=comma, y=line_8] {\ipcons};
\addplot+[boxplot={box extend=0.25, draw position=1}, ColorGenAlg, rshift2, solid, fill=ColorGenAlg!20, mark=x] table [col sep=comma, y=line_1] {\gacons};
\addplot+[boxplot={box extend=0.25, draw position=2}, ColorGenAlg, rshift2, solid, fill=ColorGenAlg!20, mark=x] table [col sep=comma, y=line_2] {\gacons};
\addplot+[boxplot={box extend=0.25, draw position=3}, ColorGenAlg, rshift2, solid, fill=ColorGenAlg!20, mark=x] table [col sep=comma, y=line_3] {\gacons};
\addplot+[boxplot={box extend=0.25, draw position=4}, ColorGenAlg, rshift2, solid, fill=ColorGenAlg!20, mark=x] table [col sep=comma, y=line_4] {\gacons};
\addplot+[boxplot={box extend=0.25, draw position=5}, ColorGenAlg, rshift2, solid, fill=ColorGenAlg!20, mark=x] table [col sep=comma, y=line_5] {\gacons};
\addplot+[boxplot={box extend=0.25, draw position=6}, ColorGenAlg, rshift2, solid, fill=ColorGenAlg!20, mark=x] table [col sep=comma, y=line_6] {\gacons};
\addplot+[boxplot={box extend=0.25, draw position=7}, ColorGenAlg, rshift2, solid, fill=ColorGenAlg!20, mark=x] table [col sep=comma, y=line_7] {\gacons};
\addplot+[boxplot={box extend=0.25, draw position=8}, ColorGenAlg, rshift2, solid, fill=ColorGenAlg!20, mark=x] table [col sep=comma, y=line_8] {\gacons};
\addplot+[boxplot={box extend=0.25, draw position=9}, ColorGenAlg, rshift2, solid, fill=ColorGenAlg!20, mark=x] table [col sep=comma, y=line_9] {\gacons};
\addplot+[boxplot={box extend=0.25, draw position=10}, ColorGenAlg, rshift2, solid, fill=ColorGenAlg!20, mark=x] table [col sep=comma, y=line_10] {\gacons};
\addplot+[boxplot={box extend=0.25, draw position=11}, ColorGenAlg, rshift2, solid, fill=ColorGenAlg!20, mark=x] table [col sep=comma, y=line_11] {\gacons};
\addplot+[boxplot={box extend=0.25, draw position=12}, ColorGenAlg, rshift2, solid, fill=ColorGenAlg!20, mark=x] table [col sep=comma, y=line_12] {\gacons};
7.07 \end{axis}
\end{tikzpicture}
\caption{CO$_{2}$ emissions for assignments using integer program (\textcolor{ColorLegendIP}{$\blacksquare$}), simulated annealing (\textcolor{ColorLegendSimAnn}{$\blacksquare$}), genetic algorithm (\textcolor{ColorLegendGenAlg}{$\blacksquare$}) and greedy algorithm (\textcolor{ColorLegendGreedy}{$\blacksquare$}).}	
\label{fig:cons_all}
\end{figure}
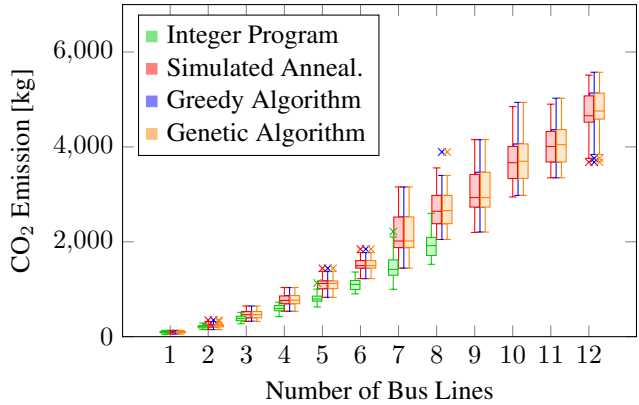

\end{revision}

\cref{fig:cons_all} shows CO$_2$ emissions for the integer program, simulated annealing, genetic algorithm and the greedy algorithm with problem instances of various sizes (i.e., number of bus lines).
For each problem size, we evaluate the algorithms on 35 random samples with the same setting as in \cref{fig:cost_all} of \cref{sec:numerical}.
We observe that the results for energy costs and CO$_2$ emissions are very similar, suggesting that we can reduce both, leading to more affordable and environment-friendly transit.

\pgfplotstableread[col sep=comma,]{data/full/co2_emission/greedy_co2_full.csv}\greedyco
\pgfplotstableread[col sep=comma,]{data/full/co2_emission/real_co2_full.csv}\realco
\pgfplotstableread[col sep=comma,]{data/full/co2_emission/sim_anneal_co2_full.csv}\saco
\pgfplotstableread[col sep=comma,]{data/full/co2_emission/gen_algo_co2_full.csv}\gaco

\begin{revision}
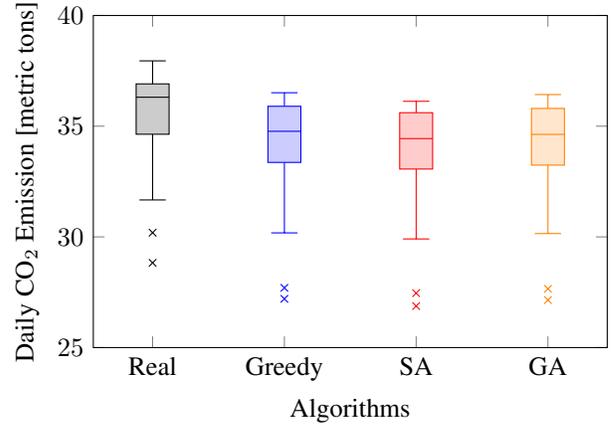
\begin{figure}[h!]
\begin{tikzpicture}
\begin{axis}[
      boxplot/draw direction=y,
      xtick={1,2,3,4},
      xticklabels={{Real}, {Greedy}, {SA}, {GA}},
      width=\columnwidth,
      height = 6cm,
      bugsResolvedStyle/.style={},
      ylabel={Daily CO$_{2}$ Emission [metric tons]},
      xlabel={Algorithms},
      ymin=25,
      ymax=40,
    ]
\addplot+[boxplot={box extend=0.25, draw position=1}, ColorReal, solid, fill=ColorReal!20, mark=x] table [col sep=comma, y=real] {\realco};
\addplot+[boxplot={box extend=0.25, draw position=2},ColorGreedy, solid, fill=ColorGreedy!20, mark=x] table [col sep=comma, y=greedy] {\greedyco};
\addplot+[boxplot={box extend=0.25, draw position=3}, ColorSimAnn, solid, fill=ColorSimAnn!20, mark=x] table [col sep=comma, y=sim_anneal] {\saco};
\addplot+[boxplot={box extend=0.25, draw position=4}, ColorGenAlg, solid, fill=ColorGenAlg!20, mark=x] table [col sep=comma, y=gen_alg] {\gaco};
\end{axis}
\end{tikzpicture}
\caption{CO$_2$ emissions for assignments using greedy algorithm (\textcolor{ColorLegendGreedy}{$\blacksquare$}), simulated annealing (\textcolor{ColorLegendSimAnn}{$\blacksquare$}) and genetic algorithm (\textcolor{ColorLegendGenAlg}{$\blacksquare$}) for complete daily schedules,
compared to existing real-world assignments (\textcolor{ColorLegendReal}{$\blacksquare$}).}
	\label{fig:co2_full_sch}
\end{figure}
\end{revision}

\cref{fig:co2_full_sch} shows CO$_2$ emissions using the greedy, simulated annealing and genetic algorithm for 50 different sample days, assigning 3 electric and 50 liquid-fuel buses to the complete daily schedule of the agency.
We use the same settings as in \cref{fig:cost_full_sch} of \cref{sec:numerical} and again compare to the existing real-world assignments.
Similar to the smaller problem instances, we observe that the results for energy costs and CO$_2$ emissions are almost identical.


\begin{revision}
\pgfplotstableread[col sep=comma,]{data/small/costs/greedy_diff_costs.csv}\greedycost
\pgfplotstableread[col sep=comma,]{data/small/costs/sim_anneal_diff_costs.csv}\simannealcost
\pgfplotstableread[col sep=comma,]{data/small/costs/gen_alg_diff_costs.csv}\gacost

\begin{revision}

\begin{figure}
\begin{tikzpicture}
\begin{axis}[
      boxplot/draw direction=y,
      xtick={1,2,3,4,5,6,7,8},
      width=\columnwidth,
      height = 5cm,
      bugsResolvedStyle/.style={},
      ylabel={Energy Cost},
      xlabel={Number of Bus Lines},
      ymin=95,
      ymax=250,
      yticklabel=\pgfmathprintnumber{\tick}\,$\%$,
      legend pos=north west,
      legend columns=3,
      legend cell align={left},
    ]
    \addlegendimage{ColorLegendSimAnn, only marks, mark=square*};
    \addlegendentry{\,Sim. Anneal.};
    \addlegendimage{ColorLegendGreedy, only marks, mark=square*}
    \addlegendentry{\,Greedy Alg.};
    \addlegendimage{ColorLegendGenAlg, only marks, mark=square*}
    \addlegendentry{\,Gen. Alg.};
    
\addplot+[boxplot={box extend=0.25, draw position=1},ColorGreedy, rshift2, solid, fill=ColorGreedy!20, mark=x] table [col sep=comma, y=line_1] {\greedycost};
\addplot+[boxplot={box extend=0.25, draw position=2},ColorGreedy, rshift2, solid, fill=ColorGreedy!20, mark=x] table [col sep=comma, y=line_2] {\greedycost};
\addplot+[boxplot={box extend=0.25, draw position=3},ColorGreedy, rshift2, solid, fill=ColorGreedy!20, mark=x] table [col sep=comma, y=line_3] {\greedycost};
\addplot+[boxplot={box extend=0.25, draw position=4},ColorGreedy, rshift2, solid, fill=ColorGreedy!20, mark=x] table [col sep=comma, y=line_4] {\greedycost};
\addplot+[boxplot={box extend=0.25, draw position=5},ColorGreedy, rshift2, solid, fill=ColorGreedy!20, mark=x] table [col sep=comma, y=line_5] {\greedycost};
\addplot+[boxplot={box extend=0.25, draw position=6},ColorGreedy, rshift2, solid, fill=ColorGreedy!20, mark=x] table [col sep=comma, y=line_6] {\greedycost};
\addplot+[boxplot={box extend=0.25, draw position=7},ColorGreedy, rshift2, solid, fill=ColorGreedy!20, mark=x] table [col sep=comma, y=line_7] {\greedycost};
\addplot+[boxplot={box extend=0.25, draw position=8},ColorGreedy, rshift2, solid, fill=ColorGreedy!20, mark=x] table [col sep=comma, y=line_8] {\greedycost};

\addplot+[boxplot={box extend=0.25, draw position=1}, ColorSimAnn, solid, fill=ColorSimAnn!20, mark=x] table [col sep=comma, y=line_1] {\simannealcost};
\addplot+[boxplot={box extend=0.25, draw position=2}, ColorSimAnn, solid, fill=ColorSimAnn!20, mark=x] table [col sep=comma, y=line_2] {\simannealcost};
\addplot+[boxplot={box extend=0.25, draw position=3}, ColorSimAnn, solid, fill=ColorSimAnn!20, mark=x] table [col sep=comma, y=line_3] {\simannealcost};
\addplot+[boxplot={box extend=0.25, draw position=4}, ColorSimAnn, solid, fill=ColorSimAnn!20, mark=x] table [col sep=comma, y=line_4] {\simannealcost};
\addplot+[boxplot={box extend=0.25, draw position=5}, ColorSimAnn, solid, fill=ColorSimAnn!20, mark=x] table [col sep=comma, y=line_5] {\simannealcost};
\addplot+[boxplot={box extend=0.25, draw position=6}, ColorSimAnn, solid, fill=ColorSimAnn!20, mark=x] table [col sep=comma, y=line_6] {\simannealcost};
\addplot+[boxplot={box extend=0.25, draw position=7}, ColorSimAnn, solid, fill=ColorSimAnn!20, mark=x] table [col sep=comma, y=line_7] {\simannealcost};
\addplot+[boxplot={box extend=0.25, draw position=8}, ColorSimAnn, solid, fill=ColorSimAnn!20, mark=x] table [col sep=comma, y=line_8] {\simannealcost};

\addplot+[boxplot={box extend=0.25, draw position=1}, ColorGenAlg, lshift2, solid, fill=ColorGenAlg!20, mark=x] table [col sep=comma, y=line_1] {\gacost};
\addplot+[boxplot={box extend=0.25, draw position=2}, ColorGenAlg, lshift2, solid, fill=ColorGenAlg!20, mark=x] table [col sep=comma, y=line_2] {\gacost};
\addplot+[boxplot={box extend=0.25, draw position=3}, ColorGenAlg, lshift2, solid, fill=ColorGenAlg!20, mark=x] table [col sep=comma, y=line_3] {\gacost};
\addplot+[boxplot={box extend=0.25, draw position=4}, ColorGenAlg, lshift2, solid, fill=ColorGenAlg!20, mark=x] table [col sep=comma, y=line_4] {\gacost};
\addplot+[boxplot={box extend=0.25, draw position=5}, ColorGenAlg, lshift2, solid, fill=ColorGenAlg!20, mark=x] table [col sep=comma, y=line_5] {\gacost};
\addplot+[boxplot={box extend=0.25, draw position=6}, ColorGenAlg, lshift2, solid, fill=ColorGenAlg!20, mark=x] table [col sep=comma, y=line_6] {\gacost};
\addplot+[boxplot={box extend=0.25, draw position=7}, ColorGenAlg, lshift2, solid, fill=ColorGenAlg!20, mark=x] table [col sep=comma, y=line_7] {\gacost};
\addplot+[boxplot={box extend=0.25, draw position=8}, ColorGenAlg, lshift2, solid, fill=ColorGenAlg!20, mark=x] table [col sep=comma, y=line_8] {\gacost};
7.07 \end{axis}
\end{tikzpicture}
\caption{Energy cost for assignments using simulated annealing (\textcolor{ColorLegendSimAnn}{$\blacksquare$}), genetic algorithm~(\textcolor{ColorLegendGenAlg}{$\blacksquare$}) and the greedy algorithm~(\textcolor{ColorLegendGreedy}{$\blacksquare$}) compared to optimal assignments (found using the integer program).}
	\label{fig:cost_all_with_ga}
\end{figure}
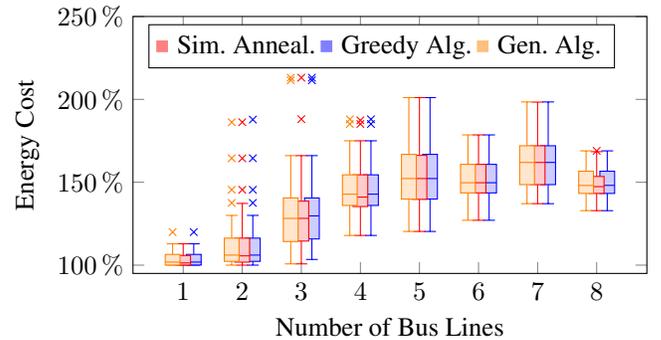

\end{revision}
\cref{fig:cost_all_with_ga} is extension to the \cref{fig:cost_all} in the main paper including the results of Genetic Algorithm. All three of approaches are not perform well compared to the IP, but the cost ratio between IP and our heuristic remains between the range of 1.5 to 1.6 even for larger instances.

\pgfplotstableread[col sep=comma,]{data/full/costs/greedy_diff_costs.csv}\greedyco
\pgfplotstableread[col sep=comma,]{data/full/costs/sim_anneal_diff_costs.csv}\saco
\pgfplotstableread[col sep=comma,]{data/full/costs/gen_algo_diff_costs.csv}\gaco

\begin{revision}
\begin{figure}
\begin{tikzpicture}
\begin{axis}[
      boxplot/draw direction=y,
      xtick={1,2,3},
      xticklabels={{Greedy}, {Sim. Ann.}, {Gen. Algo.}},
      width=\columnwidth,
      height = 4cm,
      bugsResolvedStyle/.style={},
      ylabel={Energy Cost},
      xlabel={Algorithms},
      yticklabel=\pgfmathprintnumber{\tick}\,$\%$,
      ymin=90,
      ymax=100,
    ]
\addplot+[boxplot={box extend=0.25, draw position=1},ColorGreedy, solid, fill=ColorGreedy!20, mark=x] table [col sep=comma, y=greedy] {\greedyco};
\addplot+[boxplot={box extend=0.25, draw position=2}, ColorSimAnn, solid, fill=ColorSimAnn!20, mark=x] table [col sep=comma, y=sim_anneal] {\saco};
\addplot+[boxplot={box extend=0.25, draw position=3}, ColorGenAlg, solid, fill=ColorGenAlg!20, mark=x] table [col sep=comma, y=gen_algo] {\gaco};
\end{axis}
\end{tikzpicture}
\caption{Energy costs for assignments using the greedy algorithm (\textcolor{ColorLegendGreedy}{$\blacksquare$}), simulated annealing (\textcolor{ColorLegendSimAnn}{$\blacksquare$}) and genetic algorithm (\textcolor{ColorLegendGenAlg}{$\blacksquare$}) for complete daily schedules,
compared to existing real-world assignments.}
\label{fig:cost_full_sch_with_ga}
\end{figure}
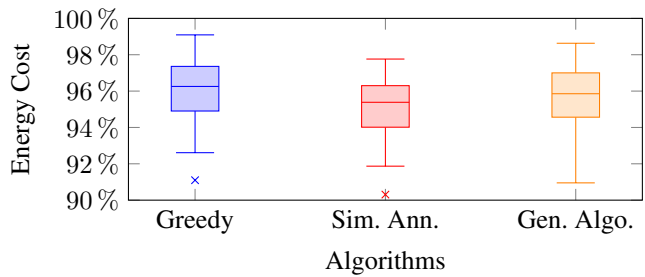
\end{revision}
\cref{fig:cost_full_sch_with_ga} is extension to the \cref{fig:cost_full_sch} in the main paper including the results of Genetic Algorithm. The results shows that though genetic algorithm slightly improves the solution obtained from the greedy algorithm, the improvement is less significant with respect to simulated annealing algorithm.
\end{revision}

\begin{revision}
\paragraph{Parameter Tuning}

\begin{revision}
\begin{table*}[!ht]
    \centering
\begin{tabular}{|c|c|c|c|c|c|c|c|c|c|c|}
    \hline
    \multicolumn{2}{|c|}{\multirow{2}{*}{}} 
    & \multicolumn{9}{c|}{Wait-time factor  of liquid-fuel buses $\alpha_\text{gas}$}
    \\\cline{3-11}
    \multicolumn{2}{|c|}{} & \bfseries 0.0001 & \bfseries 0.0002 & \bfseries 0.0003 & \bfseries 0.0004 & \bfseries 0.0005 & \bfseries 0.0006 & \bfseries 0.0007 & \bfseries 0.0008 & \bfseries 0.0009
    \\\hline
    \multirow{9}{*}{\shortstack[l]{Wait-time factor\\ of electric buses\\ $\alpha_\text{electric}$}}
&  \textbf{0.001} &     8496.6 &      8378.1 &     8574.5 &     8619.1 &     8639.2 &     8614.4 &     8603.2 &     8605.5 &     8578.0\\\cline{2-11}
&  \textbf{0.002} &     8494.9 &     8553.1 &     8583.0 &     8613.4 &     8633.1 &     8625.1 &     8605.5 &     8589.5 &     8575.7\\\cline{2-11}
&  \textbf{0.003} &     8588.1 &     8621.6 &     8661.8 &     8663.8 &     8667.8 &     8654.5 &     8669.5 &     8636.5 &     8619.0\\\cline{2-11}
&  \textbf{0.004} &     9072.7 &     8980.1 &     8986.1 &     9027.4 &     9037.4 &     8939.3 &     8945.2 &     8839.2 &     8814.6\\\cline{2-11}
&  \textbf{0.005} &     9201.5 &     9137.3 &     9133.2 &     9155.8 &     9165.1 &     9173.8 &     9105.7 &     9054.1 &     9052.1\\\cline{2-11}
&  \textbf{0.006} &     9221.1 &     9157.4 &     9172.3 &     9183.0 &     9182.7 &     9179.6 &     9150.1 &     9133.9 &     9121.3\\\cline{2-11}
&  \textbf{0.007} &     9221.1 &     9157.4 &     9172.1 &     9199.1 &     9200.2 &     9179.6 &     9167.3 &     9159.2 &     9155.5\\\cline{2-11}
&  \textbf{0.008} &     9221.1 &     9157.4 &     9180.3 &     9207.3 &     9200.2 &     9207.0 &     9188.8 &     9181.1 &     9155.5\\\cline{2-11}
&  \textbf{0.009} &     9221.1 &     9157.4 &     9180.3 &     9207.3 &     9208.4 &     9207.0 &     9188.8 &     9181.1 &     9176.1\\\hline
    \end{tabular}
        \caption{Energy costs for assigning vehicles to complete daily transit schedules using the greedy algorithm with different wait-time factors $\alpha_\text{electric}$ and $\alpha_\text{gas}$ for electric and liquid-fuel buses, respectively. Each value is the average of assignments for 6 different sample days using the full fleet of the agency.}
    \label{fig:weights_tuning}
\end{table*}
\end{revision}
\begin{revision}
\begin{table*}[!ht]
    \centering
\begin{tabular}{|c|c|c|c|c|c|c|c|c|c|c|}
    \hline
    \multicolumn{2}{|c|}{\multirow{2}{*}{}} 
    & \multicolumn{9}{c|}{Wait-time factor  of liquid-fuel buses $\alpha_\text{gas}$}
    \\\cline{3-11}
    \multicolumn{2}{|c|}{} & \bfseries 1$\cdot 10^{-5}$ & \bfseries 2$\cdot 10^{-5}$ & \bfseries 3$\cdot 10^{-5}$ & \bfseries 4$\cdot 10^{-5}$ & \bfseries 5$\cdot 10^{-5}$ & \bfseries 6$\cdot 10^{-5}$ & \bfseries 7$\cdot 10^{-5}$ & \bfseries 8$\cdot 10^{-5}$ & \bfseries 9$\cdot 10^{-5}$
    \\\hline
    \multirow{9}{*}{\shortstack[l]{Wait-time factor\\ of electric buses\\ $\alpha_\text{electric}$}}
&  \textbf{1$\cdot 10^{-5}$} &      516.9 &      517.0 &      516.3 &      516.0 &      516.1 &      516.4 &      516.5 &      517.0 &      517.0\\\cline{2-11}
&  \textbf{2$\cdot 10^{-5}$} &      513.7 &      513.5 &      512.9 &      512.7 &      513.0 &      513.4 &      513.7 &      513.8 &      513.9\\\cline{2-11}
&  \textbf{3$\cdot 10^{-5}$} &      516.3 &      516.1 &      515.5 &      515.3 &      515.5 &      516.0 &      516.2 &      516.2 &      516.3\\\cline{2-11}
&  \textbf{4$\cdot 10^{-5}$} &      516.5 &      516.1 &      515.6 &      515.6 &      515.8 &      516.5 &      516.8 &      517.0 &      517.2\\\cline{2-11}
&  \textbf{5$\cdot 10^{-5}$} &      517.8 &      517.7 &      517.9 &      517.4 &      517.5 &      517.6 &      518.2 &      518.4 &      518.5\\\cline{2-11}
&  \textbf{6$\cdot 10^{-5}$} &      517.8 &      517.7 &      517.9 &      517.4 &      517.5 &      517.6 &      518.2 &      518.4 &      518.5\\\cline{2-11}
&  \textbf{7$\cdot 10^{-5}$} &      517.8 &      517.7 &      517.9 &      517.4 &      517.5 &      517.6 &      518.2 &      518.4 &      518.5\\\cline{2-11}
&  \textbf{8$\cdot 10^{-5}$} &      517.8 &      517.7 &      517.9 &      517.4 &      517.5 &      517.6 &      518.2 &      518.4 &      518.5\\\cline{2-11}
&  \textbf{9$\cdot 10^{-5}$} &      517.4 &      517.5 &      517.7 &      517.2 &      517.3 &      517.4 &      518.0 &      518.2 &      518.3\\\hline
    \end{tabular}
        \caption{Energy costs for assigning vehicles to complete daily transit schedules using the greedy algorithm with different wait-time factors $\alpha_\text{electric}$ and $\alpha_\text{gas}$ for electric and liquid-fuel buses, respectively. Each value is the average of assignments for 6 different sample instances of 5 bus lines serving 10 trips each.}
    \label{fig:weights_small}
\end{table*}
\end{revision}

Here, we present numerical results supporting the choice of parameters values for \cref{algo:energy_cost,algo:greedy_approach,algo:mutation,algo:sti_mul_anneal} (see first paragraph of \cref{sec:results}).
We use a grid search to tune the wait-time factor ($\alpha$) parameter of the greedy algorithm for both liquid-fuel and electric buses. Note that we use different parameters for different vehicle types since the optimal parameters are different. We explored a range of values from $10^{-6}$ to $10^{6}$ for liquid-fuel buses and electric buses.
\cref{fig:weights_tuning} shows the average energy costs of assigning vehicles to complete daily transit schedules with different parameter combinations.  Each value is the average of assignments over 6 different sample days. Note that we included only values around the optima for ease of presentation.
For smaller problem instances, we explored a range of values from $10^{-15}$ to $10^{10}$  for both liquid-fuel buses and electric buses. \cref{fig:weights_small} shows the average energy costs of assigning vehicles to sample instances serving 5 bus-lines with 10 trips per bus-line with different parameter combinations. Each line in the figures is the average of 5 runs of the simulated annealing algorithms for 6 different random sample instances.

\pgfplotstableread[col sep=comma,]{data/add/p_values/1_9_5_sim_anneal_short.csv}\pointone
\pgfplotstableread[col sep=comma,]{data/add/p_values/2_9_5_sim_anneal_short.csv}\pointtwo
\pgfplotstableread[col sep=comma,]{data/add/p_values/3_9_5_sim_anneal_short.csv}\pointthree
\pgfplotstableread[col sep=comma,]{data/add/p_values/4_9_5_sim_anneal_short.csv}\pointfour
\pgfplotstableread[col sep=comma,]{data/add/p_values/5_9_5_sim_anneal_short.csv}\pointfive

\begin{revision}
\begin{figure}[!ht]
\begin{tikzpicture}[font=\footnotesize]
\begin{axis}[
    legend style={at={(0.5,-0.37)},anchor=north},
	grid=major,
	xlabel=Number of Iterations,
	ylabel={Energy Cost [\$]},
	width=\columnwidth,
	legend columns=3,
	height=6cm]

\addplot [color=black, no marks, x=cycle_count, y=energy_cost] table\pointone;
\addlegendentry{0.1}
\addplot [color=blue, no marks,x=cycle_count, y=energy_cost] table\pointtwo;
\addlegendentry{0.2}
\addplot [color=green, no marks,x=cycle_count, y=energy_Cost] table {\pointthree};
\addlegendentry{0.3}
\addplot [color=red, no marks,x=cycle_count, y=energy_Cost] table {\pointfour};
\addlegendentry{0.4}
\addplot [color=orange, no marks, x=cycle_count, y=energy_cost] table\pointfive;
\addlegendentry{0.5}
\end{axis}
\end{tikzpicture}
\caption{Energy costs for assignments using simulated annealing with different $p_\text{start}$ values (with fixed $p_\text{end}$ = 0.09 and $p_\text{swap}$ = 0.05). Each value is the average of assignments for 5 different sample days using the full fleet of the agency.}
\label{fig:p_start_tuning}
\end{figure}
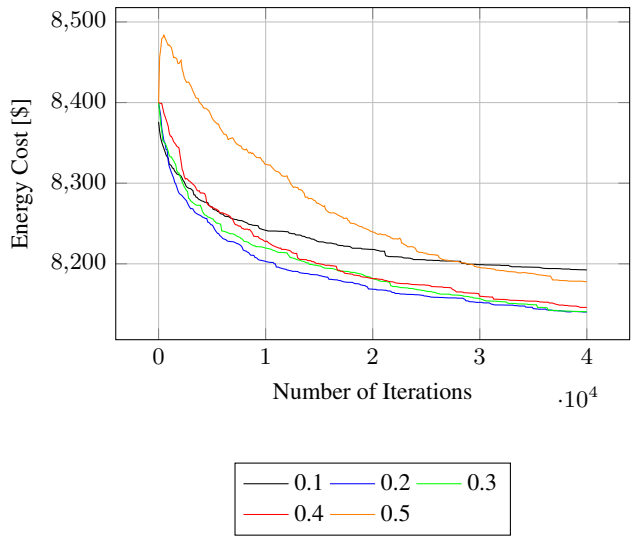
\end{revision}
\pgfplotstableread[col sep=comma,]{data/add/p_values/2_1_5_sim_anneal_short.csv}\pointone
\pgfplotstableread[col sep=comma,]{data/add/p_values/2_2_5_sim_anneal_short.csv}\pointtwo
\pgfplotstableread[col sep=comma,]{data/add/p_values/2_3_5_sim_anneal_short.csv}\pointthree
\pgfplotstableread[col sep=comma,]{data/add/p_values/2_4_5_sim_anneal_short.csv}\pointfour
\pgfplotstableread[col sep=comma,]{data/add/p_values/2_5_5_sim_anneal_short.csv}\pointfive
\pgfplotstableread[col sep=comma,]{data/add/p_values/2_6_5_sim_anneal_short.csv}\pointsix
\pgfplotstableread[col sep=comma,]{data/add/p_values/2_7_5_sim_anneal_short.csv}\pointseven
\pgfplotstableread[col sep=comma,]{data/add/p_values/2_8_5_sim_anneal_short.csv}\pointeight
\pgfplotstableread[col sep=comma,]{data/add/p_values/2_9_5_sim_anneal_short.csv}\pointnine

\begin{revision}
\begin{figure}[!ht]
\begin{tikzpicture}[font=\footnotesize]
\begin{axis}[
    legend style={at={(0.5,-0.37)},anchor=north},
	grid=major,
	xlabel=Number of Iterations,
	ylabel={Energy Cost [\$]},
	width=\columnwidth,
	legend columns=3,
	height=6cm]

\addplot [color=black, no marks, x=cycle_count, y=energy_cost] table\pointone;
\addlegendentry{0.01}
\addplot [color=blue, no marks,x=cycle_count, y=energy_cost] table\pointtwo;
\addlegendentry{0.02}
\addplot [color=green, no marks,x=cycle_count, y=energy_Cost] table {\pointthree};
\addlegendentry{0.03}
\addplot [color=red, no marks,x=cycle_count, y=energy_Cost] table {\pointfour};
\addlegendentry{0.04}
\addplot [color=orange, no marks, x=cycle_count, y=energy_cost] table\pointfive;
\addlegendentry{0.05}
\addplot [color=cyan, no marks,x=cycle_count, y=energy_cost] table\pointsix;
\addlegendentry{0.06}
\addplot [color=pink, no marks,x=cycle_count, y=energy_Cost] table {\pointseven};
\addlegendentry{0.07}
\addplot [color=magenta, no marks,x=cycle_count, y=energy_Cost] table {\pointeight};
\addlegendentry{0.08}
\addplot [color=purple, no marks,x=cycle_count, y=energy_Cost] table {\pointnine};
\addlegendentry{0.09}
\end{axis}
\end{tikzpicture}
\caption{Energy costs for assignments using simulated annealing with different $p_\text{end}$ values (with fixed $p_\text{start}$ = 0.2 and $p_\text{swap}$ = 0.05). Each value is the average of assignments for 5 different sample days using the full fleet of the agency.}
\label{fig:p_end_tuning}
\end{figure}
\end{revision}
\pgfplotstableread[col sep=comma,]{data/add/p_values/2_9_1_sim_anneal_short.csv}\pointone
\pgfplotstableread[col sep=comma,]{data/add/p_values/2_9_2_sim_anneal_short.csv}\pointtwo
\pgfplotstableread[col sep=comma,]{data/add/p_values/2_9_3_sim_anneal_short.csv}\pointthree
\pgfplotstableread[col sep=comma,]{data/add/p_values/2_9_4_sim_anneal_short.csv}\pointfour
\pgfplotstableread[col sep=comma,]{data/add/p_values/2_9_5_sim_anneal_short.csv}\pointfive

\begin{revision}

\begin{figure}[!ht]
\begin{tikzpicture}[font=\footnotesize]
\begin{axis}[
    legend style={at={(0.5,-0.37)},anchor=north},
	grid=major,
	xlabel=Number of Iterations,
	ylabel={Energy Cost [\$]},
	width=\columnwidth,
	legend columns=3,
	height=6cm]

\addplot [color=black, no marks, x=cycle_count, y=energy_cost] table\pointone;
\addlegendentry{0.01}
\addplot [color=blue, no marks,x=cycle_count, y=energy_cost] table\pointtwo;
\addlegendentry{0.02}
\addplot [color=green, no marks,x=cycle_count, y=energy_Cost] table {\pointthree};
\addlegendentry{0.03}
\addplot [color=red, no marks,x=cycle_count, y=energy_Cost] table {\pointfour};
\addlegendentry{0.04}
\addplot [color=orange, no marks, x=cycle_count, y=energy_cost] table\pointfive;
\addlegendentry{0.05}
\end{axis}
\end{tikzpicture}
\caption{Energy costs for assignments using simulated annealing with different $p_\text{swap}$ values (with fixed $p_\text{start}$ = 0.2 and $p_\text{swap}$ = 0.05). Each value is the average of assignments for 5 different sample days using the full fleet of the agency.}
\label{fig:p_swap_tuning}
\end{figure}

\end{revision}

\pgfplotstableread[col sep=comma,]{data/add/p_s_values/1_7_3_sim_anneal_short.csv}\pointone
\pgfplotstableread[col sep=comma,]{data/add/p_s_values/2_7_3_sim_anneal_short.csv}\pointtwo
\pgfplotstableread[col sep=comma,]{data/add/p_s_values/3_7_3_sim_anneal_short.csv}\pointthree
\pgfplotstableread[col sep=comma,]{data/add/p_s_values/4_7_3_sim_anneal_short.csv}\pointfour
\pgfplotstableread[col sep=comma,]{data/add/p_s_values/5_7_3_sim_anneal_short.csv}\pointfive
\pgfplotstableread[col sep=comma,]{data/add/p_s_values/6_7_3_sim_anneal_short.csv}\pointsix
\pgfplotstableread[col sep=comma,]{data/add/p_s_values/7_7_3_sim_anneal_short.csv}\pointseven
\pgfplotstableread[col sep=comma,]{data/add/p_s_values/8_7_3_sim_anneal_short.csv}\pointeight
\pgfplotstableread[col sep=comma,]{data/add/p_s_values/9_7_3_sim_anneal_short.csv}\pointnine

\begin{revision}

\begin{figure}[!ht]
\begin{tikzpicture}[font=\footnotesize]
\begin{axis}[
    legend style={at={(0.5,-0.37)},anchor=north},
	grid=major,
	xlabel=Number of Iterations,
	ylabel={Energy Cost [\$]},
	width=\columnwidth,
	legend columns=3,
	height=6cm]

\addplot [color=black, no marks, x=cycle_count, y=energy_cost] table\pointone;
\addlegendentry{0.1}
\addplot [color=blue, no marks,x=cycle_count, y=energy_cost] table\pointtwo;
\addlegendentry{0.2}
\addplot [color=green, no marks,x=cycle_count, y=energy_Cost] table {\pointthree};
\addlegendentry{0.3}
\addplot [color=red, no marks,x=cycle_count, y=energy_Cost] table {\pointfour};
\addlegendentry{0.4}
\addplot [color=orange, no marks, x=cycle_count, y=energy_cost] table\pointfive;
\addlegendentry{0.5}
\addplot [color=cyan, no marks,x=cycle_count, y=energy_cost] table\pointsix;
\addlegendentry{0.6}
\addplot [color=pink, no marks,x=cycle_count, y=energy_Cost] table {\pointseven};
\addlegendentry{0.7}
\addplot [color=magenta, no marks,x=cycle_count, y=energy_Cost] table {\pointeight};
\addlegendentry{0.8}
\addplot [color=purple, no marks,x=cycle_count, y=energy_Cost] table {\pointnine};
\addlegendentry{0.9}
\end{axis}
\end{tikzpicture}
\caption{Energy costs for assignments using simulated annealing with different $p_\text{start}$ values (with fixed $p_\text{end}$ = 0.07 and $p_\text{swap}$ = 0.03). Each value is the average of assignments for 6 different sample instances of 5 bus lines serving 10 trips each.}
\label{fig:p_start_tuning_each}
\end{figure}
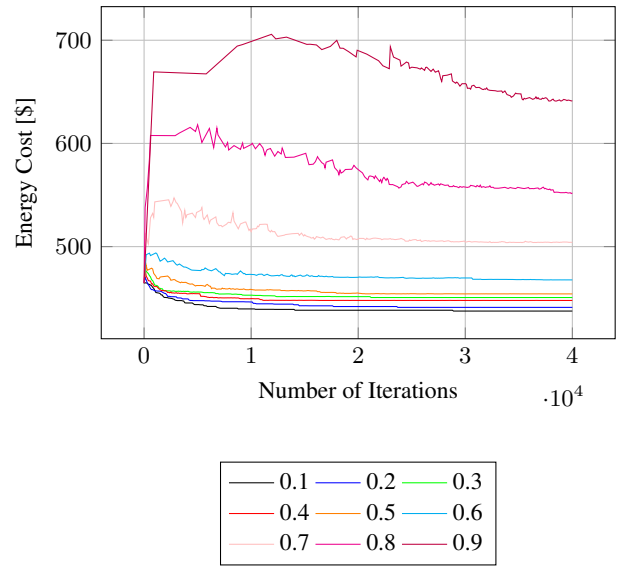

\end{revision}
\pgfplotstableread[col sep=comma,]{data/add/p_s_values/1_1_3_sim_anneal_short.csv}\pointone
\pgfplotstableread[col sep=comma,]{data/add/p_s_values/1_2_3_sim_anneal_short.csv}\pointtwo
\pgfplotstableread[col sep=comma,]{data/add/p_s_values/1_3_3_sim_anneal_short.csv}\pointthree
\pgfplotstableread[col sep=comma,]{data/add/p_s_values/1_4_3_sim_anneal_short.csv}\pointfour
\pgfplotstableread[col sep=comma,]{data/add/p_s_values/1_5_3_sim_anneal_short.csv}\pointfive
\pgfplotstableread[col sep=comma,]{data/add/p_s_values/1_6_3_sim_anneal_short.csv}\pointsix
\pgfplotstableread[col sep=comma,]{data/add/p_s_values/1_7_3_sim_anneal_short.csv}\pointseven
\pgfplotstableread[col sep=comma,]{data/add/p_s_values/1_8_3_sim_anneal_short.csv}\pointeight
\pgfplotstableread[col sep=comma,]{data/add/p_s_values/1_9_3_sim_anneal_short.csv}\pointnine

\begin{revision}

\begin{figure}[!ht]
\begin{tikzpicture}[font=\footnotesize]
\begin{axis}[
    legend style={at={(0.5,-0.37)},anchor=north},
	grid=major,
	xlabel=Number of Iterations,
	ylabel={Energy Cost [\$]},
	width=\columnwidth,
	legend columns=3,
	height=6cm]

\addplot [color=black, no marks, x=cycle_count, y=energy_cost] table\pointone;
\addlegendentry{0.01}
\addplot [color=blue, no marks,x=cycle_count, y=energy_cost] table\pointtwo;
\addlegendentry{0.02}
\addplot [color=green, no marks,x=cycle_count, y=energy_Cost] table {\pointthree};
\addlegendentry{0.03}
\addplot [color=red, no marks,x=cycle_count, y=energy_Cost] table {\pointfour};
\addlegendentry{0.04}
\addplot [color=orange, no marks, x=cycle_count, y=energy_cost] table\pointfive;
\addlegendentry{0.05}
\addplot [color=cyan, no marks,x=cycle_count, y=energy_cost] table\pointsix;
\addlegendentry{0.06}
\addplot [color=pink, no marks,x=cycle_count, y=energy_Cost] table {\pointseven};
\addlegendentry{0.07}
\addplot [color=magenta, no marks,x=cycle_count, y=energy_Cost] table {\pointeight};
\addlegendentry{0.08}
\addplot [color=purple, no marks,x=cycle_count, y=energy_Cost] table {\pointnine};
\addlegendentry{0.09}
\end{axis}
\end{tikzpicture}
\caption{Energy costs for assignments using simulated annealing with different $p_\text{end}$ values (with fixed $p_\text{start}$ = 0.1 and $p_\text{swap}$ = 0.03). Each value is the average of assignments for 6 different sample instances of 5 bus lines serving 10 trips each.}
\label{fig:p_end_tuning_each}
\end{figure}
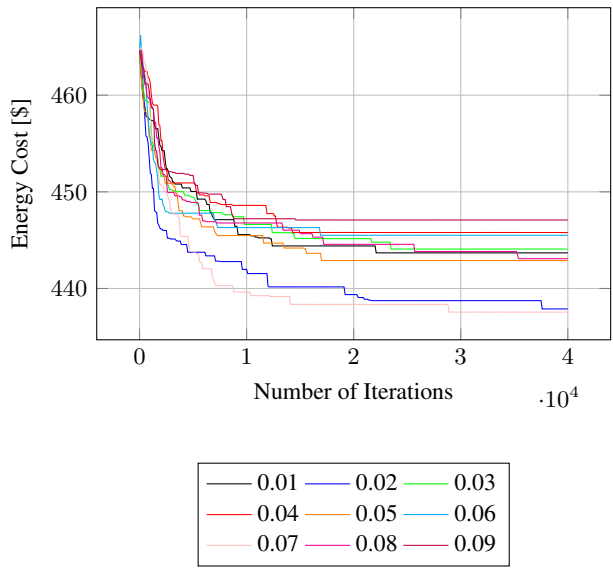

\end{revision}
\pgfplotstableread[col sep=comma,]{data/add/p_s_values/1_7_1_sim_anneal_short.csv}\pointone
\pgfplotstableread[col sep=comma,]{data/add/p_s_values/1_7_2_sim_anneal_short.csv}\pointtwo
\pgfplotstableread[col sep=comma,]{data/add/p_s_values/1_7_3_sim_anneal_short.csv}\pointthree
\pgfplotstableread[col sep=comma,]{data/add/p_s_values/1_7_4_sim_anneal_short.csv}\pointfour
\pgfplotstableread[col sep=comma,]{data/add/p_s_values/1_7_5_sim_anneal_short.csv}\pointfive
\pgfplotstableread[col sep=comma,]{data/add/p_s_values/1_7_6_sim_anneal_short.csv}\pointsix
\pgfplotstableread[col sep=comma,]{data/add/p_s_values/1_7_7_sim_anneal_short.csv}\pointseven
\pgfplotstableread[col sep=comma,]{data/add/p_s_values/1_7_8_sim_anneal_short.csv}\pointeight
\pgfplotstableread[col sep=comma,]{data/add/p_s_values/1_7_9_sim_anneal_short.csv}\pointnine
\begin{revision}

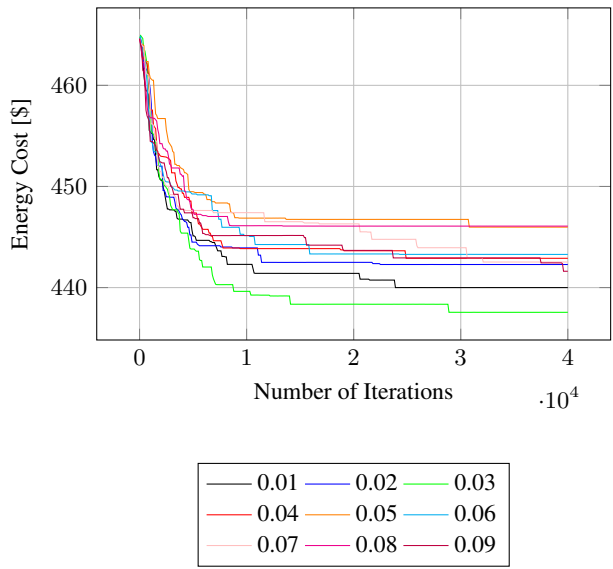
\begin{figure}[!ht]
\begin{tikzpicture}[font=\footnotesize]
\begin{axis}[
    legend style={at={(0.5,-0.37)},anchor=north},
	grid=major,
	xlabel=Number of Iterations,
	ylabel={Energy Cost [\$]},
	width=\columnwidth,
	legend columns=3,
	height=6cm]

\addplot [color=black, no marks, x=cycle_count, y=energy_cost] table\pointone;
\addlegendentry{0.01}
\addplot [color=blue, no marks,x=cycle_count, y=energy_cost] table\pointtwo;
\addlegendentry{0.02}
\addplot [color=green, no marks,x=cycle_count, y=energy_Cost] table {\pointthree};
\addlegendentry{0.03}
\addplot [color=red, no marks,x=cycle_count, y=energy_Cost] table {\pointfour};
\addlegendentry{0.04}
\addplot [color=orange, no marks, x=cycle_count, y=energy_cost] table\pointfive;
\addlegendentry{0.05}
\addplot [color=cyan, no marks,x=cycle_count, y=energy_cost] table\pointsix;
\addlegendentry{0.06}
\addplot [color=pink, no marks,x=cycle_count, y=energy_Cost] table {\pointseven};
\addlegendentry{0.07}
\addplot [color=magenta, no marks,x=cycle_count, y=energy_Cost] table {\pointeight};
\addlegendentry{0.08}
\addplot [color=purple, no marks,x=cycle_count, y=energy_Cost] table {\pointnine};
\addlegendentry{0.09}
\end{axis}
\end{tikzpicture}
\caption{Energy costs for assignments using simulated annealing with different $p_\text{swap}$ values (with fixed $p_\text{start}$ = 0.1 and $p_\text{end}$ = 0.07). Each value is the average of assignments for 6 different sample instances of 5 bus lines serving 10 trips each.}
\label{fig:p_swap_tuning_each}
\end{figure}

\end{revision}

Similarly, we exhaustively searched the parameters space of $p_\text{start}$ for a range of values from 0.1 to 0.5, $p_\text{end}$ for a range of values from 0.01 to 0.09, $p_\text{swap}$ for a range of values from 0.01 to 0.05. \cref{fig:p_start_tuning} shows how the energy cost evolves over the iterations of the simulated annealing algorithm with  various $p_\text{start}$ values.
Each line is the average of 5 runs of the simulated annealing algorithms for 5 different sample days. Note that these results also show that with the right parameters, simulated annealing converges in around 40,000 iterations; hence, it suffices to use $k_{max} = $ 50,000 for our numerical results. For smaller problem instances, we exhaustively searched the parameters space of $p_\text{start}$ for a range of values from 0.1 to 0.9, $p_\text{end}$ for a range of values from 0.01 to 0.9, $p_\text{swap}$ for a range of values from 0.01 to 0.09. \cref{fig:p_start_tuning_each} shows how the energy cost evolves over the iterations of the simulated annealing algorithm with  various $p_\text{start}$ values. \cref{fig:p_end_tuning_each} shows how the energy cost evolves over the iterations of the simulated annealing algorithm with  various $p_\text{end}$ values.
\cref{fig:p_swap_tuning_each} shows how the energy cost evolves over the iterations of the simulated annealing algorithm with  various $p_\text{swap}$ values.
Each line in the figures is the average of 5 runs of the simulated annealing algorithms for 5 different random sample instances serving 5 bus-lines with 10 trips per bus-line. Same as the complete daily schedule, these results also show that with the right parameters, simulated annealing converges in around 40,000 iterations.
\end{revision}

\newpage
\paragraph{Length of Time Slots}
\pgfplotstableread[col sep=comma,]{data/add/slot_duration/slot_duration.csv}\intprog

\begin{figure}[ht]
\begin{tikzpicture}[font=\footnotesize]
\begin{axis}[
    legend style={at={(0.5,-0.37)},anchor=north},
    ymax=50,
    ymin=10,
	grid=major,
	xlabel={Length of Time Slot [hours]},
	ylabel={Energy Cost [\$]},
	width=\columnwidth,
	legend columns=2,
	height=4cm]
\addplot [color=magenta, mark=*, select coords between index={0}{2}] table [x=slot_duration, y=ip_cost,] {\intprog};
\addplot [color=black, mark=x, select coords between index={3}{5}] table [x=slot_duration, y=ip_cost,] {\intprog};
\legend{1 Bus Line, 2 Bus Lines}
\end{axis}
\end{tikzpicture}
\caption{Energy costs for assignments using the integer program with different time-slot lengths.}
\label{fig:time_slot}
\end{figure}
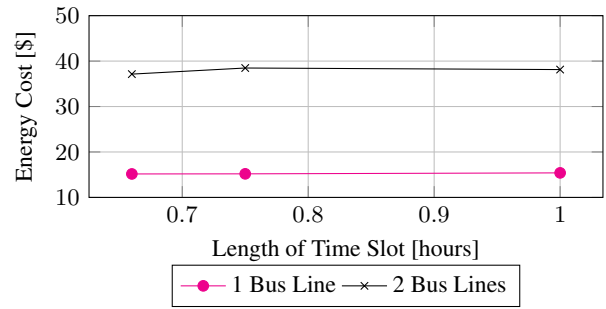

In \cref{sec:model}, we discretized charging schedules into uniform-length time slots for the sake of computational tractability.
Now, we study whether discretizing time has a significant impact on solution quality by comparing assignments with various time-slot lengths. 
Since the integer program can find optimal solutions for small instances, we study the impact of time-slot lengths using the integer program for  1 or 2 bus lines with 10 trips for each line.
\cref{fig:time_slot} shows energy costs for various settings.
Each plotted value is the average taken over 3 different sample instances.
The figure shows that the loss in solution quality is very small even with longer time slots, such as $1$ hour.

\fi

\end{document}